\documentclass{article}

\usepackage[english]{babel}

\usepackage{fullpage}

\usepackage{amsmath,amsthm, amssymb}
\usepackage{algorithm}
\usepackage{algorithmic}
\usepackage{graphicx}
\usepackage[colorlinks=true, allcolors=blue]{hyperref}
\usepackage{commath}
\usepackage{dsfont} 
\usepackage{xspace} 
\usepackage{natbib}
\usepackage[dvipsnames]{xcolor}
\usepackage{mathtools}

\renewcommand{\Pr}{\field{P}}

\newcommand{\bb}{\boldsymbol{b}}

\newcommand{\bx}{\boldsymbol{x}}
\newcommand{\bu}{\boldsymbol{u}}

\newcommand{\bw}{\boldsymbol{w}}

\newcommand{\argmax}{\mathop{\mathrm{argmax}}}

\newcommand{\field}[1]{\mathbb{#1}}

\newcommand{\R}{\field{R}}

\newcommand{\E}{\field{E}}

\newcommand{\dt}{\displaystyle}

\newcommand{\sign}{{\rm sign}}

\DeclareMathOperator{\Regret}{Regret}
\DeclareMathOperator{\Wealth}{Wealth}

\newtheorem{lemma}{Lemma}
\newtheorem{proposition}{Proposition}
\newtheorem{theorem}{Theorem}

\newtheorem{remark}{Remark}

\defcitealias{JunO19}{Jun and Orabona [COLT'19]}

\newcommand{\fr}[2]{ { \frac{#1}{#2} }}
\def\lam{\ensuremath{\lambda}}
\def\cd{\cdot}
\def\dt{\delta}

\def\kl{\mathsf{D}}
\DeclareMathOperator{\PP}{\mathbb{P}}
\def\hmu{{\hat\mu}}
\definecolor{kjcolor}{RGB}{46,139,87}

\DeclareMathOperator{\one}{\mathds{1}\hspace{-.18em}}
\def\onec#1{\one\cbr{#1}}
\def\RR{{\mathbb{R}}}
\def\der{\ensuremath{\partial}\xspace}
\def\tsty{\textstyle}

\def\til{\tilde}

\def\hbeta{{\hat\beta}}

\newcommand{\sr}[2]{ {\stackrel{#1}{#2}} }
\def\th{{\ensuremath{\theta}}}

\def\barV{{\bar{V}}}

\title{Tight Concentrations and Confidence Sequences\\ from the Regret of Universal Portfolio}
\author{Francesco Orabona\\Department of Electrical \& Computer Engineering\\Boston University\\\texttt{francesco@orabona.com} \and Kwang-Sung Jun\\Department of Computer Science\\University of Arizona\\\texttt{kjun@cs.arizona.edu}}

\begin{document}
\maketitle

\begin{abstract}
    A classic problem in statistics is the estimation of the expectation of random variables from samples. This gives rise to the tightly connected problems of deriving concentration inequalities and confidence sequences, that is confidence intervals that hold uniformly over time.
    Previous work have shown how to easily convert the regret guarantee of an online betting algorithm into a time-uniform concentration inequality.
    In this paper, we show that we can go even further: We show that the regret of universal portfolio algorithms give rise to new implicit time-uniform concentrations and state-of-the-art empirically calculated confidence sequences.
    In particular, our numerically obtained confidence sequences can never be vacuous, even with a single sample, and satisfy the law of iterated logarithm.
\end{abstract}

\section{Introduction}
\label{sec:intro}

Constructing valid confidence intervals for the estimation of bounded random variables is an important problem in many machine learning and statistical problems.
In particular, we are interested in estimating the conditional mean $\mu$ of a sequence of random variables $X_1, X_2, \dots$, such that for any $i$ we have $0 \leq X_i \leq 1$ and 
\begin{equation}
\label{eq:assumption}
\E[X_t | X_1, \dots, X_{t-1}] = \mu~.
\end{equation}
A simple example of the condition above is obtained when $X_i$ are independent random variables with mean equal to $\mu$.

We are interested in proving time-uniform empirical concentration inequalities that upper bounds the probability that the empirical mean deviates from the expectation $\mu$.
A related concept is the one of producing  $(1-\delta)$-level \emph{confidence sequences}~\citep{DarlingR67}, that is a sequence of confidence sets $I_t = I_t(X_1,\dots,X_t)$ such that
\[
\Pr\{\mu \in I_t, \forall t \geq1\}\geq 1 -\delta~.
\]
We are interested in $I_t$ being intervals, that is $I_t=[\ell_t, u_t]$. The time-uniform property makes these estimates more useful in the real-world because they are ``safer'' to be used, allowing, for example, to stop gathering new samples at any time, even in an adaptive way, when the required precision has been reached. This was even recently argued in much more general terms by \citet{GrunwaldHK19}.

In general, it is possible to obtain a confidence sequence from a time-uniform empirical concentration inequality. However, often this will result in confidence sequences that are vacuous when the number of samples is small. 

In this paper, for practical usability, we are interested in obtaining numerically tight concentrations and confidence sequences. 
That is, we are interested in rates, but we also care about constants, in order to obtain the tightest possible estimates for real-world applications, as recently motivated by \citet{PhanTLM21}.  In particular, we aim at obtaining confidence sequences with width less than 1 even with a single sample and at the same time optimal asymptotic behaviour.

We build on \citet{JunO19} that showed any online betting algorithm with small regret can be used to derive a time-uniform concentration inequality. In this paper, we expand over their approach and show how to obtain state-of-the art time-uniform concentrations and confidence sequences through the regret guarantee of portfolio selection algorithms.
In particular, our contributions are as follows:
\begin{itemize}
    \item We derive a new implicit concentration inequality for bounded random variables satisfying \eqref{eq:assumption}.
    \item We numerically invert our implicit concentration, obtaining empirical results that match or outperform prior algorithms. In particular, our algorithm is the first one to give non-vacuous confidence sequence for \textit{every} number of samples (Algorithm~\ref{alg:confidence_intervals}: PRECiSE-CO96). 
    \item We also show how to symbolically invert it, which implies that our concentration inequality is a time-uniform empirical Bernstein type inequality.
    \item We also show how to obtain a fast numerical inversion with $O(1)$ computational complexity per sample while enjoying the same asymptotic performance of the exact inversion (Algorithm~\ref{alg:confidence_intervals_a}: PRECiSE-A-CO96).
    \item Finally, we show how to obtain state-of-the-art confidence sequences that satisfy the law of the iterated logarithm (Algorithm~\ref{alg:confidence_intervals_lil}: PRECiSE-R70).
\end{itemize}
The rest of the paper is organized as follows: In Section~\ref{sec:rel} we discuss prior work, explaining similarities and differences. In Section~\ref{sec:def} we precisely define some concepts of online learning. Our main results are presented in Section~\ref{sec:main} and~\ref{sec:lil} with their proofs deferred to Section~\ref{sec:proofs}. Finally, in Section~\ref{sec:exp} we present our numerical evaluation and we conclude in Section~\ref{sec:disc} with a discussion.

\section{History and Related Work}
\label{sec:rel}

As noted in \citet{GrunwaldHK19}, the interest on this topic exploded in the last couple of years. However, recent studies seem to have missed the fact that these most of these ideas seem to have been discovered in the information theory literature.
Hence, we find very important to give here not only a discussion of related work, but also a detailed history of these ideas.

\paragraph{Betting and Portfolio Selection Algorithms}
The distributional approach to betting and gambling was pioneered by \citet{Kelly56}. This approach assumes that the market gains are i.i.d. from a (known) distribution. \citet{Kelly56} also noted the equivalence between gambling and prediction with log loss.
Moreover, given the well-known connection between universal compression, prediction with log loss, and gambling~\citep{CoverO96}, prior literature on compression and prediction with log loss \citep[see, e.g.,][]{Shtarkov87, Vovk90, HausslerB92, FederMG92, XieB97, HausslerKW98} effectively deals with gambling too. In particular, the best achievable regret for a betting game with a known number of rounds on a coin was calculated by \citet{Shtarkov87}. A suboptimal betting strategy based on the Laplace rule of succession to estimate the bias of the coin when the number of rounds is unknown is given in \cite{Davisson73}.  Asympotic optimality in the stochastic case was obtained with a mixture based on Jeffrey's prior by \citet{KrichevskyT81} and analyzed in finite-time in \citet{BarronX96}.
\citet{Freund96,KumonTT08,KumonTT11} removed the distributional assumptions and focused on the problem of betting on a bounded outcome (even in multiple dimensions). Later, the online setting, i.e., without distributional assumptions, became standard in the universal compression literature.

\citet{McMahanO14} showed that it is possible to design online betting algorithm with a $\ln T$ regret against a weaker competitor.
\citet{OrabonaP16} introduced a generic potential framework to design and analyze online algorithms for the same problem. That paper started a line of research on coin-betting algorithms with efficient updates~\citep[e.g.,][]{CutkoskyO18,MhammediK20}, but an outsider to this subfield might have missed the fact that \emph{all of them} are known to be suboptimal. 
Indeed, none of them guarantee any regret for betting on a \emph{continuous} coin $g_t \in [-1,1]$ against the best constant betting fraction. 
For example, the known regret bound (e.g., \cite[Section 4]{OrabonaP16}) of the KT estimator is relative to the betting fraction $\frac{1}{t}\sum_{i=1}^t g_i$, which is the best constant betting fraction for a coin rather than that of the continuous coin.
In fact, a minimax optimal algorithm for continuous coins is already known.
To see this, one can use the folklore knowledge that symmetric betting on a bounded symmetric outcome can be easily reduced to portfolio selection with two stocks.
In turn, the portfolio algorithm of \citet{CoverO96}, which is the first one to achieve the minimax regret without assumptions over the market gains, an improved result upon \citet{Cover91}.
(Note that \citet{VovkW98} proved that the universal portfolio algorithm \citep{Cover91,CoverO96} is an instantiation of the Aggregating Algorithm~\citep{Vovk90}.)
On a related note, betting on an asymmetric bounded outcome can be also reduced to portfolio selection with 2 stocks, as we show later. 
Lastly, note that the family of online betting algorithms is larger than one might think. Indeed, \citet[Theorem 5.11]{Orabona19} proved that \emph{any} online linear optimization algorithm is equivalent to a betting algorithm with initial money equal to the regret against the null competitor.

Regarding algorithms for solving the coin-betting problem, most of the computationally efficient coin-betting algorithms are based on Follow-The-Regularized-Leader~\citep{Shalev-Shwartz07,AbernethyHR08,HazanK08}.
Unfortunately, the portfolio algorithm by \citet{CoverO96} does not have an update rule with constant computational complexity, though it can be approximated~\citep{KalaiV02}.

\paragraph{Testing and betting algorithms}
The connection between betting and testing is very old. 
The first work on the connection between betting strategies and probability is in the PhD thesis of \citet{Ville39}. He proved that an event depending on an infinite sequence of outcomes has probability zero iff we can achieve infinite wealth betting on the outcomes of the sequence. In particular, he defined a \emph{martingale} as the wealth of a betting strategy on a fair coin. In \citet[page 34]{Ville39} there is also the very first appearance of the wealth process of a betting strategy based on the Laplace method of counts, more than 30 years before \citet{Davisson73} used it in coding. However, as \citet[page 431]{Shafer21} points out, ``Ville [...] did not consider statistical testing, and his work has had little or no influence on mathematical statistics.''
The ideas of Ville were later used to design ideal tests for randomness of infinite sequences~\citep{Schnorr71,Levin76,Gacs05}, ``ideal'' because none of these tests is computable.

We found two papers that at the same time and independently explicitly link hypothesis testing on a finite sequence of outcomes to betting. One is \citet[Example 3]{Cover74}, where an optimal betting strategy is defined in terms of the null hypothesis.
The other one is \citet{RobbinsS74} that constructs confidence sequences from novel\footnote{Very surprisingly, it seems that nobody realized before that \citet[Section 9]{RobbinsS74} seem to have proposed and analyzed the famous strategy of \citet{KrichevskyT81} 7 years before them.} betting schemes and explicitly recognizes the connection between the sequential probability ratio test~\citep{Wald45}.
However, while in the information theory literature these ideas flourished and gave birth to results on coding, compression, minimum description length, and gambling, they seem to disappear from the statistics community for 30 years.\footnote{It is remarkable that \citet{Cover74} was submitted to Annals of Statistics and probably rejected, as it can be inferred from the footnote on its first page.} One might argue that any result on martingales is still related to gambling, but, even assuming that this view is correct, it ignores the fact that the gambling view stresses aspects like the computability of the strategies, their optimality, the adversarial nature of the outcomes, and the gambler's wealth as evidence, that are missing from the literature on martingales. 

In fact, in the statistics literature gambling strategies reappear again only in the '90-'00s thanks to the book and papers by Shafer and Vovk. In particular, \citet{Vovk93,ShaferV05} aimed to found probabilities on a game-theoretic ground through betting schemes. 
They also proposed the idea of using non-negative martingales as \emph{test martingales} \citep{ShaferSVS11,Shafer21}.
However, the foundational approach in \citet{ShaferV05} also means that all the betting strategies do not have closed form expressions and they cannot be easily implemented.
Moreover, \citet{ShaferV05} does not contain any explicit concentration, while the first concentration for game-theoretic probability derived by a betting scheme is in \citet{Vovk07}, that derives a game-theoretic Hoeffding’s inequality.

Instead, in the information theory literature, the use of compression schemes to test randomness of a finite string of symbols became a standard strategy~\citep{Ziv90,Maurer92,Rukhin00,RyabkoM05}, even with software available by the American National Institute of Standards and Technology~\citep{BasshamETAL10}. Given the connection between compression and gambling~\citep{CoverO96}, these approaches can also be considered as tests based on betting.

Even if the idea of numerically deriving valid confidence sequences from any betting algorithm does not appear in print in the work of Shafer and Vovk, it is an ``obvious'' corollary of their results.\footnote{G. Shafer, 2022, personal communication.}
That said, the first paper to consider an implementable strategy for testing through betting is by \citet{Hendriks18}. Directly building on \citet{ShaferV05} and \citet{ShaferSVS11}, \citet{Hendriks18} proposes to construct testing martingales and confidence sequences for bounded random variables as uniform mixtures of constant betting strategies. \citet{Hendriks18} also showed empirically the good performance of the proposed approach in a simple statistical test. 
\citet{JunO19} show how to easily derive a Law of Iterated Logarithm (LIL) for sub-Gaussian random vectors in Banach spaces from the regret of a one-dimensional betting algorithm.
It is important to note that the betting scheme they propose has a very high computational complexity, but they only need any upper bound to the regret of the algorithm in order to derive the concentration.
In turn, their work was based on the seminal work in \citet{RakhlinS17} that showed an \emph{equivalence} between the regret guarantee of online learning algorithms with linear losses and concentration inequalities. However, the proof technique in \citet{JunO19} is different from the one in \citet{RakhlinS17} and it is specific to online algorithms that guarantee a non-negative exponential wealth for biased inputs. In particular, it allows to derive time-uniform concentrations, like the law of the iterated logarithm, that are not possible with the method in \citet{RakhlinS17}.
\citet{Waudby-SmithR21} seems to follow the same approach in \citet{Hendriks18}, but proposing a number of heuristic betting algorithms to maximize the wealth as well as a discrete version of the uniform mixture already used in \citet{Hendriks18}. They show very good empirical results and asymptotic rates of the obtained confidence sequences. It is worth noting that the idea of designing heuristics gambling schemes goes back at least to \citet{AdachiT11}, that used neural networks, and it is present in the online learning literature too, see \citet{LiH18} for a recent review.
As in previous work~\citep[see, e.g.,][]{CutkoskyO18}, \citet{Waudby-SmithR21} guarantee non-negative wealth for their betting heuristics restricting the range of the allowed betting fractions, so that the algorithm can lose in a round at most a fixed user-defined fraction of the current wealth, e.g., 1/2 or 2/3. Also, they need the heuristic betting schemes to satisfy strict conditions in order to guarantee that the confidence sets are intervals. 
On the computational complexity side, their algorithms require to run a number of betting schemes in parallel, which results in running the betting algorithm $O(\frac{1}{\text{precision}})$ times. 

Notably, \citet{Waudby-SmithR21} explicitly motivate the heuristic betting with the idea that regret-based algorithms tend to underperform in practice ``because regret bounds could be tight in rate but are typically loose in their constants'' and ``the resulting concentration inequalities are not tight in practice''~\citep[Appendix D]{Waudby-SmithR21}.
Here, we argue the opposite: The use of non-regret based betting algorithm is a step back. Indeed, there is no need to use heuristic schemes in the hope to have a low computational complexity because it is straightforward to construct ideal betting algorithms based on portfolio strategies and evaluate tight lower bounds to their wealth using a regret analysis. Moreover, the presence of a regret guarantee will allow us to construct confidence sequences with $O(1)$ complexity per sample.

\paragraph{Principles for Constructing Betting Algorithms}
It is worth noting that, a part form \citet{JunO19}, \emph{none} of the above described work proposed a unifying general way to construct betting strategies. 
In fact, even recently \citet[page 424]{Shafer21} writes ``How  should  the  statistician  choose  the  strategy  for  Skeptic?  An  obvious  goal  is  to  obtain  small  warranty sets. But a strategy that produces the smallest warranty set for one  $N$  and one warranty level $1/\alpha$  will not generally do so for other values of these parameters. So any choice will be a balancing act. How to perform this balancing act is an important topic for further research.'' 

Actually, in 2019 two papers independently propose a similar idea as a guiding principle for betting strategies for testing. \citet{JunO19} proposed to use betting algorithms designed to minimize the regret with respect to the best constant betting strategy in hindsight for the specific application of confidence sequences. On the other hand, \citet{GrunwaldHK19} proposed to construct betting strategies for testing composite, but parametric, null hypotheses that maximize the expected log wealth. The connection between the two approaches is due to the fact that the expected log wealth is maximized by a constant betting strategy~\citep{Breiman61}, but the approach in \citet{GrunwaldHK19} does not seem to be immediately extendable to the non-parametric setting.
Moreover, we will explicitly address the ``balancing act'' mentioned by \citet{Shafer21}, obtaining tight confidence sequences with one sample or at infinity.

\paragraph{Method of mixtures in concentration and betting proofs} A powerful approach to prove time-uniform concentration inequality is through the use of the so-called method of mixtures~\citep{Robbins70}. In the simplest case, the key idea is to use a sequence of zero-mean random variables $Y_t \in \R$ to construct a martingale as $\int \exp( \sum_{i=1}^t (\lambda Y_i - \frac{\lambda^2}{2} Y_i^2)) d F(\lambda)$, for a mixture distribution $F$. Then, use Ville's inequality to infer that $\sum_{i=1}^t Y_i$ is controlled by high probability.
However, the very same idea is also the core one for universal compression~\citep{Shtarkov87} and betting algorithms~\citep{Cover74,Cover91}, where the betting algorithm is constructed as a mixture of betting schemes that bet fixed fraction of the current wealth. The connection is far from being accidental: \citet{ShaferV05,ShaferV19} explain us how the concept of martingale and the one of wealth of a bettor can be completely unified. More in details, for bounded random variables the expression above is nothing else than a second-order Taylor approximation of the growth rate of the wealth for a $F$-weighted portfolio algorithm, see Section~\ref{sec:def} for an exact definition.

\section{Setting and Definitions}
\label{sec:def}

\paragraph{Online learning}

The online learning framework is a learning scenario that proceeds in rounds, without probabilistic assumptions. We refer the readers to \citet{Cesa-BianchiL06} and \citet{Orabona19} for complete introductions to this topic.
Briefly, in each round the learner outputs a prediction $\bx_t$ in a feasible set $V$, then the adversary reveals a loss function $\ell_t: V \to \R^d$ from a fixed set of functions, and the algorithm pays $\ell_t(\bx_t)$. Aim of the learner is to minimize its \emph{regret} defined as the difference between its cumulative loss over $T$ rounds and the one of any fixed point $\bu \in V$, that is
\[
\Regret_T(\bu) := \sum_{t=1}^T \ell_t(\bx_t) - \sum_{t=1}^T \ell_t(\bu)~.
\]
Note that we do not assume anything on how the adversary chooses the functions $\ell_t$ nor any knowledge of the future. We will say that an algorithm has no-regret for a certain class of functions if its regret is $o(T)$ for any $\bu \in V$.
In the following, the keyword \emph{online} will always denote a instantation of the above framework.

\paragraph{Online betting on a continuous coin}
We consider a gambler making repeated bets on the outcomes of adversarial coin flips. The gambler starts with initial
money of \$1. In each round $t$, he bets on the outcome of a coin
flip $c_t \in \{-1,1\}$, where $+1$ denotes heads and $-1$ denotes tails.  We
do not make any assumption on how $c_t$ is generated. The gambler can bet any amount on either heads or tails. However, he is not allowed to borrow any additional money. If he loses, he loses the betted
amount; if he wins, he gets the betted amount back and, in addition to that, he
gets the same amount as a reward.  We encode the gambler's bet in round $t$ by
a single number $x_t$. The sign of $x_t$ encodes whether he is betting on heads
or tails. The absolute value encodes the betted amount.  We define $\Wealth_t$
as the gambler's wealth at the end of round $t$, that is
\begin{align}
\label{eq:def_wealth_reward}
\Wealth_t = 1 + \sum_{i=1}^t x_i c_i~.
\end{align}
In the following, we will also refer to a bet with $\beta_t$, where $\beta_t$
is such that
\begin{equation}
\label{equation:def_wt}
x_t = \beta_t \Wealth_{t-1}~.
\end{equation}
The absolute value of $\beta_t$ is the \emph{fraction} of the current wealth to
bet and its sign encodes whether he is betting on heads or tails. The
constraint that the gambler cannot borrow money implies that $\beta_t \in
[-1,1]$.
We also slightly generalize the problem by allowing the outcome of the coin flip
$c_t$ to be any real number in $[-a,b]$ where $a,b>0$. Following \citet{OrabonaP16}, we will call this case \emph{continuous coin}; the definition of the wealth remains the same.

From the above, it should be immediate that we can define an equivalent online learning game where the algorithm produces the signed fraction $\beta_t \in [-\frac{1}{b},\frac{1}{a}]$ and the adversary reveals the convex loss function $\ell_t(\beta) = -\ln(1+\beta c_t)$, where $c_t \in [-a,b]$. In this case, $\sum_{t=1}^T \ell_t(\beta_t)$ is the negative log wealth of the algorithm and the regret is the logarithm of the ratio between the wealth of the constant betting fraction $u$ and the wealth of the algorithm.

\paragraph{Online portfolio selection}
We now describe the problem of sequential investment in a market with 2 stocks. The behavior of the market is specified by non-negative market gains vectors $\bw_1, \dots, \bw_T$, each of them in $[0,+\infty)^2$. The coordinates of the market gains vectors represent the ratio between closing and opening price for 2 stocks. An investment strategy is specified by a vector $\bb_t$ in $B=\{[b_t, 1-b_t] \in \R^2: 0\leq b_t\leq 1\}$, and its elements specifies the fraction of the wealth invested on each stock at time $t$. The wealth of the algorithm at the end of round $t$, given by $\prod_{i=1}^t \bx_t^\top\bb_t$, will be denoted by $\Wealth_t$ and we set $\Wealth_0=1$.

As in the coin-betting problem, we can define the regret for this problem as the difference between the log wealth of the best constant rebalanced portfolio minus the log wealth of the algorithm. Denote by $\Wealth_t(\bb)$ the wealth of the \textit{constant rebalanced portfolio} with allocation $\bb$, we have
\[
\Regret_T = \max_{\bb \in B} \ \ln \Wealth_T(\bb) - \ln \Wealth_t~.
\]
We will say that a portfolio algorithm is \emph{universal} if the regret against any sequence of market gain vectors is sublinear in the time step.

We focus on the  $F$-weighted portfolio algorithms~\citep{Cover91} that output at each step
\[
\bb_t = \frac{\int_{B} \bb \Wealth_{t-1}(\bb) \, d F(\bb)}{\int_{B} \Wealth_{t-1}(\bb) \, d F(\bb)},
\]
where $F$ is a distribution over the 2-stocks.
In words, we predict with a normalized average of the expected wealth of each $\bb$ vectors according to the distribution $F$. It is easy to see that for $F$-weighted portfolio algorithms the wealth of the algorithm has a closed form expression:
\begin{equation}
\label{eq:closed_form_wealth}
\Wealth_t=\int_{B} \Wealth_{t}(\bb) \, d F(\bb)~.
\end{equation}

\section{A Never-Vacuous Concentration Inequality from a Portfolio Algorithm}
\label{sec:main}

Here, we present our new concentration inequality. It holds uniformly over time for any number of random variables, but it has an implicit form.
Then, we will show both how to invert it deriving an explicit but loose concentration inequality and with an exact numerical algorithm.
\begin{theorem}
\label{thm:main}
Let $\delta \in (0,1)$. Assume $ X_1, X_2, \dots$ a sequence of random variables such that for each $i$ we have $0 \leq X_i \leq 1$ and
\[
\E[X_i| X_1, \dots, X_{i-1}]=\mu~.
\]
Then, we have
\[
\Pr\left\{\max_t \max_{\beta \in [-\frac{1}{1-\mu},\frac{1}{\mu}]}\sum_{i=1}^t \ln\left(1+ \beta (X_i-\mu)\right) - R_t\geq \ln\frac{1}{\delta} \right\} \leq \delta,
\]
where $R_t = f(b^*_t,\lfloor t b^*_t+0.5\rfloor,t)$,
the function $f$ is defined as 
\begin{equation}
\label{eq:f}
f(b,k,t)=\ln \frac{\pi \, b^k (1-b)^{t-k} \Gamma(t+1)}{\Gamma(k+1/2) \Gamma(t-k+1/2)},
\end{equation}
and $b^*_t = \mu (1-\mu) \left(\argmax_{\beta \in [-\frac{1}{1-\mu},\frac{1}{\mu}]}\sum_{i=1}^t \ln\left(1+ \beta (X_i-\mu)\right)\right)+\mu$.
Moreover, $R_t\leq \ln\frac{\sqrt{\pi}\Gamma(t+1)}{\Gamma(t + \frac 1 2)}$.
\end{theorem}
Next, we show that while it might be cryptic, Theorem~\ref{thm:main} has a very simple interpretation.

\paragraph{Derivation of the concentration} Here, we sketch a derivation of the concentration, while we leave the details in Section~\ref{sec:proof_main}.
\citet[Section 7.2]{JunO19} showed that the regret guarantee of betting algorithms that realize a supermartingale can be used to prove concentration inequalities. \citet{JunO19} were interested in sub-exponential random variables in Banach space, but here we can distill their result in the following general theorem. 
\begin{theorem}
\label{thm:ville}
Let $\delta \in (0,1]$ and $ X_1, X_2, \dots$ a sequence of random variables
such that $\E[X_i| X_1, \dots, X_{i-1}]=0$.
Assume that we have an online algorithm that observes $X_1, \dots, X_{t-1}$ and produces a non-negative $W_t$ such that $W_0=1$ and
\[
\E[W_t|X_1, \dots, X_{i-1}]\leq W_{t-1}~.
\]
Then, we have
\[
\Pr\left\{\max_t W_t \geq \frac{1}{\delta} \right\} \leq \delta~.
\]
\end{theorem}
This is just a restatement of Ville's maximal inequality\footnote{$\Pr\{\max_t M_t \geq 1/\delta\} \leq
\delta$ for any non-negative super-martingale starting from $E[M_0] = 1$ and any $\delta \in (0, 1]$.}~\citep[page 84]{Ville39}. However, here we want to stress two important, but simple, details: i) the random variables $W_t$ are generated by an online algorithm, and ii) $W_t$ can be anything, \emph{not only} the wealth of the algorithm. Indeed, multiple cases are possible.
In particular, if $X_i$ are bounded random variables in a known range, $W_t$ can indeed be defined as the wealth of an online betting algorithm that guarantees a non-negative wealth. However, the above it is still true for any $W_t$ that lower bounds the wealth of the algorithm. Another example is the sub-exponential case, where non-negativity of the wealth cannot be guaranteed anymore. In this case, \citet{JunO19} and \citet{vanderHoeven19}, under different assumptions, showed that it is possible to design online betting algorithms such that $W_t$ can be defined as a non-negative lower bound to the \emph{expected} wealth of the algorithm.

Here, we focus on the bounded case only and $W_t$ will be \emph{a lower bound} to the wealth of an online betting algorithm.
Why a lower bound? Because evaluating the wealth of a betting strategy might be computationally expensive. 
Hence, we leverage a simple observation: \emph{We do not actually need to know the plays of the betting algorithm. Instead, we only need a lower bound to its wealth at each time step.} In particular, we can obtain a lower bound to the wealth from the optimal wealth and an upper bound to the regret of the betting algorithm.

So, we are now ready to instantiate Theorem~\ref{thm:ville}: we need some coin outcomes, a betting strategy, and its regret. As in \citet{JunO19}, we simply set the coin outcomes to $X_t-\mu$. One might object that we do not know $\mu$, but to derive the concentration, as we show below, it suffices to act as if we knew it. As betting algorithm we will use an $F$-weighted portfolio algorithm. Putting all together, it should be clear now that the concentration is just an application of Theorem~\ref{thm:ville} on a lower bound to the wealth of the algorithm obtained through the optimal wealth ($\max_{\beta \in [-\frac{1}{1-\mu},\frac{1}{\mu}]} \ \sum_{i=1}^t \ln\left(1+ \beta (X_i-\mu)\right)$) minus an upper bound to the regret of the algorithm ($R_t$).

Next, we will motivate why this is an interesting concentration and which properties it has.

\paragraph{Wealth Upper-Bounds the KL Divergence}
It may not seem obvious if the log wealth is a better candidate for constructing a confidence sequence than some standard ones like Bernoulli KL-divergence based bound~\cite[e.g.,][Theorem 10]{garivier11kl}, which works for random variables supported in [0,1]:
\begin{align*}
    \PP\left\{\max_{t} \ t\cdot \kl(\hmu_t, \mu) - \ln f(t) \ge \ln \frac{1}{\delta}\right\} \le \delta,
\end{align*}
where $\kl(p,q)=p \ln \frac{p}{q}+(1-p)\ln\frac{1-p}{1-q}$ and for some $f(t)$ that grows polynomially in $t$ or slower.
In the following proposition, we show that the maximum wealth is never worse than the KL divergence, which supports a viewpoint that KL divergence is a special case of the maximum wealth and that confidence bounds constructed with the maximum wealth is tighter that those with KL divergence.
\begin{proposition}
\label{prop:kl}
    Let $X_1,\dots,X_t \in [0,1]$, $\hmu_t = \fr1t \sum_{i=1}^t X_i$, and $\mu \in [0,1]$. 
    Then,
    \begin{align*}
        \max_{\beta \in [-\frac{1}{1-\mu},\frac{1}{\mu}]} \ \sum_{i=1}^t \ln\left(1+ \beta (X_i-\mu)\right) \ge t\cdot \kl\left(\hmu_t,\mu\right),
    \end{align*}
    where we achieve the equality if $X_i,\dots,X_t \in \{0,1\}$.
\end{proposition}
\begin{proof}
    The proof is inspired by~\citet[Lemma 7]{OrabonaP16}. Using Jensen's inequality, we have for any $X\in[0,1]$ that
    \begin{align*}
      \ln(1 + \beta(X - \mu))
        &= \ln[X(1 + \beta(1-\mu)) + (1-X) (1+\beta(0-\mu))]
      \\&\ge X \ln(1 + \beta(1-\mu)) + (1-X)\ln (1 + \beta(0-\mu))~,
    \end{align*}
    Note that we achieve the equality when $X = 1$ or $X = 0$.
    Then, we have
    \begin{align*}
      \max_{\beta \in [-\frac{1}{1-\mu},\frac{1}{\mu}]} \ \sum_{i=1}^t \ln(1 + \beta (X_i - \mu))
      &\ge \max_{\beta \in [-\frac{1}{1-\mu},\frac{1}{\mu}]} \sum_i X_i \ln(1 + \beta(1-\mu)) + (1-X_i)\ln (1 - \beta \mu)
      \\&= \max_{\beta \in [-\frac{1}{1-\mu},\frac{1}{\mu}]} t[ \hmu_t \ln(1 + \beta(1-\mu)) + (1-\hmu_t)\ln (1 - \beta\mu)  ]~.
    \end{align*}
    As the RHS is concave in $\beta$, it remains to maximize the RHS over $\beta$.
    The solution is $\beta = \fr{\hmu_t - \mu}{\mu(1-\mu)}$ with which the maximum becomes $t \cd \kl(\hmu_t, \mu)$.
\end{proof}
We verify Proposition 1 in Figure~\ref{fig:wealth_vs_kl}.
\begin{figure}
\begin{tabular}{cccc}
  \includegraphics[width=0.16\linewidth]{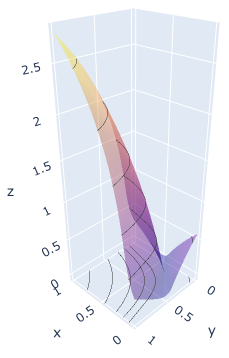} &
  \includegraphics[width=0.16\linewidth]{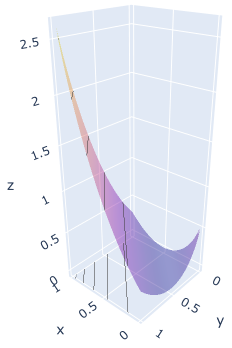} &
  \includegraphics[width=0.25\linewidth]{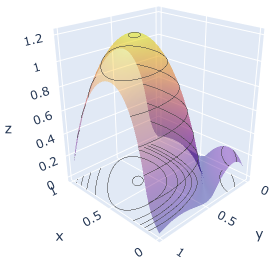} &
  \includegraphics[width=0.38\linewidth]{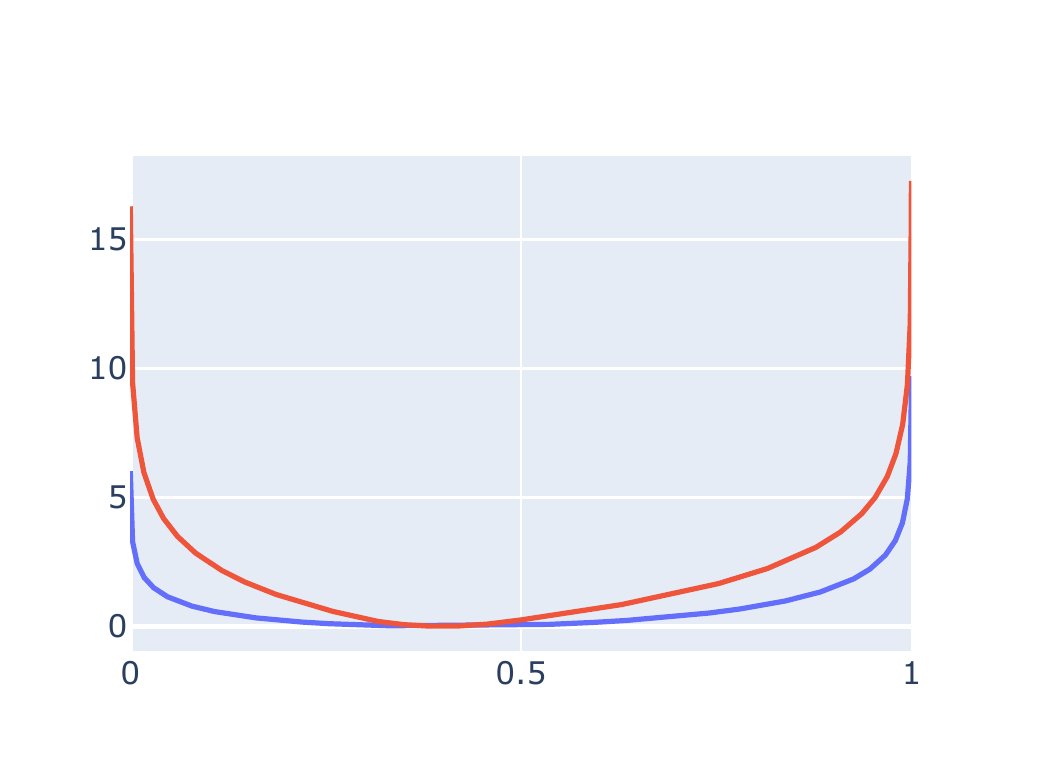} 
  \\ (a)  & (b) & (c)  & (d) 
\end{tabular}
\caption{Illustration of the fact that the maximum log wealth w.r.t. coins centered at $m$ upperbounds the KL divergence $t\cd\kl(\hat\mu_t,m)$ where $t=2$. We set $m=\fr14$. (a-c) the maximum log wealth, the KL divergence, and the value of (a) minus (c) as a function of two data points $x$ and $y$. (d) The maximum wealth (red) and the KL divergence $t\cd\kl(\hat\mu_t,m)$ (blue) as a function of $m$ with two samples $\{0.2,0.6\}$.  }
\label{fig:wealth_vs_kl}
\end{figure}

\paragraph{$F$-Weighted portfolios:  Optimal Regret, Permutation-Invariance, and Intervals} 
Theorem~\ref{thm:main} is obtained using the regret of the $F$-weighted portfolio algorithm. However, one might wonder if better algorithms are possible. This is readily answered: \citet{CoverO96} showed that $F$-weighted portfolios achieve an anytime (i.e., time-uniform) regret upper bound that is at most $\ln \sqrt{2 \pi}$ larger than the optimal achievable regret for this problem. Also, in Theorem~\ref{thm:lil} we will show that it is possible to derive the law of iterated logarithm using an $F$-weighted portfolio algorithm. Hence, $F$-weighted portfolios are optimal in two different ways. However, we can say even more.
$F$-weighted portfolios are invariant to the order of the market gains~\citep[Proposition 4]{Cover91}, as it is clear from \eqref{eq:closed_form_wealth}. Hence, the confidence interval at time $t$ obtained from the wealth of an $F$-weighted portfolio algorithm are independent of the order of the samples $X_1,\dots,X_t$. We believe this is an important property to avoid a ``lottery'' based on the ordering of the samples, as motivated in \citet{MeinshausenMB09}.
More importantly, a little known\footnote{We were unable to find any mention of this result in prior work.} result by \citet[pages 85--87]{Ville39} proves that in the case of betting on a coin \emph{the only betting strategies that are invariant to the order of the outcomes are $F$-weighted portfolios}.
Last, we can show in the following theorem that the confidence sequences derived by the wealth of a $F$-weighted portfolio algorithm in \eqref{eq:closed_form_wealth} are always intervals.
\begin{theorem}
\label{thm:f_weighted_intervals}
Let $c_i \in [0,1]$. Then, the function $F(m)=\int_{0}^{1} \prod_{i=1}^t(1+\beta(c_i-m))d F(\beta)$ is convex in $m \in [0,1]$.
\end{theorem}
\begin{proof}
For a fixed $\beta \in [-1,1]$, define $G(m)=\prod_{i=1}^t (1+\beta (c_i-m))=\exp(\sum_{i=1}^t\ln(1+\beta(c_i-m)))$.
We have that
\[
G''(m)
= \exp\left(\sum_{i=1}^t\ln(1+\beta(c_i-m))\right)\left[ \left(\sum_{i=1}^t\frac{-\beta}{1+\beta(c_i-m)}\right)^2 - \sum_{i=1}^t\frac{\beta^2}{(1+\beta(c_i-m))^2}\right] \geq 0~. 
\]
Hence, $F(m)$ is convex.
\end{proof}

In the next section, we show that Theorem~\ref{thm:main} implies an explicit concentration inequality.

\subsection{Numerical inversion of the Theorem~\ref{thm:main}}

Inverting the concentration in Theorem~\ref{thm:main} is actually immediate: For each $t$, it is enough to find all values of $m$ that satisfies the inequality 
\begin{equation}
\label{eq:to_be_inverted}
\max_{\beta \in [-\frac{1}{1-m},\frac{1}{m}]} \ \sum_{i=1}^t \ln\left(1+ \beta (X_i-m)\right) - R_t\leq \ln\frac{1}{\delta}~.
\end{equation}
This, unfortunately, might not lead to an interval.
However, we show below that there exists a tight upper bound on $R_t$ that helps us derive an interval.
This means that we can look for the extremes of the interval with a simple and efficient binary search procedure. 
In turn, the maximization over $\beta$ is one-dimensional concave problem that can be easily solved with any standard convex optimization algorithm.
The overall procedure is described in Algorithm~\ref{alg:confidence_intervals} that we call PRECiSE-CO96 (Portfolio REgret for Confidence SEquences (PRECiSE) using \citet{CoverO96}).
This algorithm is guaranteed to generate confidence sequences, as detailed in the following theorem.

\begin{algorithm}[t]
\begin{algorithmic}[1]
\STATE{\textbf{Input:} $\delta \in (0,1)$, random variables $X_1, X_2, \dots$ in $[0,1]$}
\STATE{$\ell_0=0, u_0=1$}
\FOR{$t=1,2,\dots$}
\STATE{$\hat{\mu}_t=\frac{1}{t} \sum_{i=1}^t X_t$}
\STATE{$b^{\ell}=\argmax_{b \in [0,1]} \  \sum_{i=1}^t \ln\left(b \frac{X_i}{\ell_{t-1}}+(1-b)\frac{1-X_i}{1-\ell_{t-1}}\right)$}
\STATE{$R^{\ell}_t=\max(f(\lceil b^{\ell}\cdot t-0.5\rceil/t,\lceil b^{\ell}\cdot t-0.5\rceil,t),f(\lfloor \hat{\mu}\cdot t+0.5\rfloor/t,\lfloor \hat{\mu}\cdot t+0.5\rfloor),t)$ where $f$ is in \eqref{eq:f}}
\STATE{$L_t=\{ m \in [\ell_{t-1},\hat{\mu}] : \max_{b \in [0,1]}\sum_{i=1}^t \ln\left(b \frac{X_i}{m}+(1-b)\frac{1-X_i}{1-m}\right) - R^\ell_t\geq \ln\frac{1}{\delta}\}$}
\STATE{Find $\ell_t=\max ~( L_t \cup \{\ell_{t-1}\} )$ by binary search}
\STATE{$b^{u}=\argmax_{b \in [0,1]} \  \sum_{i=1}^t \ln\left(b \frac{X_i}{u_{t-1}}+(1-b)\frac{1-X_i}{1-u_{t-1}}\right)$}
\STATE{$R^u_t=\max(f(\lfloor b^{u}\cdot t+0.5\rfloor/t,\lfloor b^{u}\cdot t+0.5\rfloor,t),f(\lceil \hat{\mu}\cdot t-0.5\rceil/t,\lceil \hat{\mu}\cdot t-0.5\rceil,t))$ where $f$ is in \eqref{eq:f}}
\STATE{$U_t=\{ m \in [\hat{\mu},u_{t-1}] : \max_{b \in [0,1]}\sum_{i=1}^t \ln\left(b \frac{X_i}{m}+(1-b)\frac{1-X_i}{m}\right) - R^u_t\geq \ln\frac{1}{\delta}\}$}
\STATE{Find $u_t=\min ~(U_t \cup \{u_{t-1}\})$ by binary search}
\ENDFOR
\end{algorithmic}
\caption{PRECiSE-CO96: Portfolio REgret for Confidence SEquences using \citet{CoverO96}}
\label{alg:confidence_intervals}
\end{algorithm}

\begin{theorem}
\label{thm:algo}
Let $\delta \in (0,1)$. Assume $ X_1, X_2, \dots$ a sequence of random variables such that for each $i$ we have $0 \leq X_i \leq 1$ and
\[
\E[X_i| X_1, \dots, X_{i-1}]=\mu~.
\]
Run Algorithm~\ref{alg:confidence_intervals} on $X_1, X_2, \dots$.
Then, for any $t$ we have that, with probability at least $1-\delta$, $\mu \in [\ell_t, u_t]$.
\end{theorem}
The algorithm works by finding for each new sample $X_t$ a lower bound $\ell_t$ and an upper bound $u_t$.
In particular, it has to find the the zero of a quasi-convex function in lines 8 and 12, that can be done by binary search. 
However, the function itself is defined as the maximum of a concave function on a bounded domain, that again can be solved by a binary search or any convex optimization algorithm. 
If we use the binary search for both with a given target numerical precision, the computational complexity for $t$-th time step is $O(t \log^2(\fr{1}{\text{precision}}))$.
We defer the full proof to Section~\ref{sec:proof_thm_algo}.

\paragraph{Never-Vacuous Guarantees}
We also prove our claim that the confidence sequences from Algorithm~\ref{alg:confidence_intervals} are never vacuous, the proof is in Section~\ref{sec:proof_thm_never_vacuous}.
\begin{theorem}
\label{thm:never_vacuous}
Under the assumptions of Theorem~\ref{thm:algo}, for any $t$ we have that $u_t-\ell_t \leq u_1-\ell_1=1-\frac{\delta}{2}$.
\end{theorem}
While people are used to reason on the asymptotics of confidence intervals, it is less known what is the right behavior for a small number of samples. 
Hence, to better understand the statement in the last theorem, it is useful to compare it with the width of the exact confidence intervals for a Binomial distribution~\citep{BlythS83}.
It is easy to see that with one observation the exact confidence width at probability $1-\delta$ is exactly $1-\delta$. 
Of course, it would be impossible to match this bound with time-uniform results. 
However, we essentially match the right dependency on $\delta$, allocating $\frac12$ one half of the probability of error on $t=1$ and the other half on all the other $t$; see Figure~\ref{fig:exact} for comparing ours with the exact confidence bound.

\begin{remark}[The regret is worst-case?]
There is a general misconception about regret upper bounds. In particular, many people tend to think that regret upper bounds are loose, pessimistic, and in general do not capture the actual performance of an online algorithm on non-adversarial sequences.
This is wrong for a number of reasons. First, in some cases it is even possible to design online learning algorithms whose regret is the \emph{same} on any sequence. This is indeed the case for betting on a coin with known number of rounds $T$~\citep[Section 9.4]{Cesa-BianchiL06}. However, this does not mean that the performance of the algorithm does not depend on the sequence itself: The regret is only the difference between the performance of the optimal predictor in a certain class and the performance of the algorithm. So, even if the regret is the same on any sequence, the performance of the optimal predictor \emph{will depend} on the sequence and so the performance of the online algorithm we are analyzing.
Another important point is that there is often a trade-off between the mathematical simplicity of a regret upper bound and its tightness, but nothing prevents us to obtain tighter and more complex expressions. For example, in the case of betting algorithms, the wealth of the algorithm in the proofs is often lower bounded by a much smaller quantity,\footnote{\citet{ShaferV19} call it ``slackening''.} yet this step is not mandatory.
Last point is the fact that a good regret upper bound will typically be tight for a particular sequence of inputs. For example, the regret upper bound of universal portfolio is tight for Kelly markets~\citep{CoverO96}.
\end{remark}

\begin{remark}[Better betting algorithms for stochastic coins?]
The above approach does not seem to use the fact that the $X_t$ are stochastic. In particular, it might not be immediately clear why we should use investing algorithms that try to match the performance of the optimal rebalanced portfolio. However, this issue is discussed at length in \citet[Chapter 16]{CoverT06}, where it is proved that, if the market gains are i.i.d., the optimal growth rate of the wealth is achieved by a constant rebalanced portfolio.
\end{remark}

\paragraph{Explicit Empirical Bernstein Time-Uniform Concentration.}
The concentration in Theorem~\ref{thm:main} is an implicit one and it may not be obvious if the induced implicit confidence bound is tight orderwise. 
It would be easy to show that it implies an asymptotic concentration, but instead we prefer to directly derive a new empirical Bernstein time-uniform concentration from it.
Our theorem below shows that the bound of Theorem~\ref{thm:main} indeed induces an orderwise tight empirical Bernstein-type inequality.
\begin{theorem}
  \label{thm:approx_inversion}
  Let $\delta \in (0,1)$. Assume $ X_1, X_2, \dots$ a sequence of random variables such that for each $i$ we have $0 \leq X_i \leq 1$ and
  \[
  \E[X_i| X_1, \dots, X_{i-1}]=\mu~.
  \]
  Denote by $\hat{\mu}_i = \frac{1}{i}\sum_{j=1}^i X_j$, $V_i = \sum_{j=1}^i (X_j - \hat{\mu}_i)^2$, $R_i =  \ln\frac{\sqrt{\pi}\Gamma(i+1)}{\delta\Gamma(i + \frac 1 2)}$, and
  \[
  \epsilon_i = \frac{4/3 i R_i +\sqrt{16/9 i^2 R_i^2+8 V_i R_i(i^2-2 i R_i)}}{2 i^2-4i R_i}~.
  \]
  Then, with probability at least $1-\delta$ uniformly for all $t$ such that $t> 2 R_t$, we have
  \[
  \max_{i=1,\dots,t} \hat{\mu}_i - \epsilon_i \leq \mu \leq \min_{i=1,\dots,t} \hat{\mu}_i + \epsilon_i~.
  \]
\end{theorem}
We defer the proof to Section~\ref{sec:proof_thm_approx}

For $t$ big enough the deviation is roughly $\frac{4/3 \ln(\sqrt{t}/\delta)}{t}+\frac{\sqrt{2 V_t \ln(\sqrt{t}/\delta)}}{t}$. Similar inequalities appears in \citet{AudibertMC09,MaurerP09}.

From the proof of Theorem~\ref{thm:approx_inversion} it should be clear that the resulting confidence sequences will match the exact ones when $t$ goes to infinity, while it will be loose for small samples sizes. 
In particular, it is easy to see that concentration is vacuous when the sample size is small. This is due to the fact that we use an approximation to the logarithm that is correct when the confidence sequences shrink to 0. 
However, as we have shown in Theorem~\ref{thm:never_vacuous}, the confidence sequence implied by Theorem~\ref{thm:main} is never vacuous.

\subsection{Computationally Efficient Version}
\label{sec:fast}

One downside of Algorithm~\ref{alg:confidence_intervals} is that the per-time-step time complexity is $\Omega(t)$.
While existing confidence bounds with empirical variance can achieve $O(1)$ per-time-step complexity, they are usually not as tight as PRECiSE-CO96 in small sample regime.
Could we obtain a confidence bound that is both tight and computationally efficient?

We answer this question affirmatively by proposing a computationally efficient version called
PRECiSE-A-CO96 (PRECiSE with Approximation using \citet{CoverO96}) described in Algorithm~\ref{alg:confidence_intervals_a}.
The main idea starts from the fact that the first and second moments can be computed incrementally.
If we could compute a tight lower bound on the log wealth $\max_{\beta\in[-1/(1-m),1/m]} \sum_{i=1}^t \ln(1 + \beta(X_i - m))$ as a function of those moments in a closed form, then we will not have to pay time complexity linear in $t$.
We leverage the tight inequality by~\citet{FanGL15} to obtain a closed-form lower bound on the maximum log wealth: $G^\ell_t(\beta,m)$ for computing the lower confidence bound $\ell_t$ and $G^u_t(\beta,m)$ for computing upper confidence bound $u_t$, both defined in Algorithm~\ref{alg:confidence_intervals_a}.
These lower approximations are maximized at the optimal betting of $\tilde\beta_\ell(m)$ and  $\tilde\beta_u(m)$ respectively for which we have a closed form as well.
This alone, however, can be loose for small $t$.
Fortunately, Proposition~\ref{prop:kl} says that the KL divergence is a valid lower bound of the maximum log wealth.
We can therefore take the maximum of these two approximations as a tight lower bound to the maximum log wealth.
Note that the algorithm uses binary search, which results in $O(\ln(1/\text{precision}))$ per-time-step time complexity if one updates $\hat\mu_t$ and $V_t$ incrementally.

\begin{algorithm}[t]
\begin{algorithmic}[1]
\STATE{\textbf{Input:} $\delta \in (0,1)$, random variables $X_1, X_2, \dots$ in $[0,1]$}
\STATE{$\ell_0=0, u_0=1$}
\FOR{$t=1,2,\dots$}
\STATE{$R_t= \ln\frac{\sqrt{\pi}\Gamma(t+1)}{\Gamma(t + \frac 1 2)}$}
\STATE{$\hat{\mu}_t=\frac{1}{t} \sum_{i=1}^t X_t$ and $\barV_t = \frac1t \sum_{i=1}^{t}(X_t - \hat{\mu}_t)^2$ (use online updates)}
\STATE{Let $G^\ell_t(\beta, m) := \beta(\hat{\mu}_t - m) - (-\ln(1-\beta m) -\beta m)\fr{(\barV_t + (\hat{\mu}_t - m)^2)}{m^2}$ and $\til\beta_\ell(m) = \frac{\hat{\mu}_t - m}{m(\hat{\mu}_t - m) + \barV_t + (\hat{\mu}_t-m)^2}  $}
\STATE {Compute $m_\ell = \min\cbr{m \in [0,\hat{\mu}_t]: \max\{G^\ell_t(\til\beta_\ell(m), m), \kl(\hat{\mu}_t, m)\} \le \fr1t \ln(e^{R_t}/\dt)}$ with binary search}
\STATE {$\ell_t = \max\cbr{m_\ell, \ell_{t-1}}$}
\STATE {Let $G^u_t(\beta, m) = -\beta(m-\hat{\mu}_t) - (-\ln(1+(1-m)\beta) + (1-m)\beta) \cd \fr{\barV_t+(m-\hat{\mu}_t)^2}{(1-m)^2} $ and $\til\beta_u(m) = \frac{-(m-\hat{\mu}_t)}{(1-m)(m-\hat{\mu}_t) + \barV_t + (m - \hat{\mu}_t)^2} $}
\STATE {Compute $m_u = \max\cbr{m \in [\hat{\mu}_t, 1]: \max\{G^u_t(\til\beta_u(m), m), \kl(\hat{\mu}_t,m)    \}\le \fr1t \ln(e^{R_t}/\dt)}$ with binary search}
\STATE {$u_t = \min\cbr{m_u, u_{t-1}}$}
\STATE {Output $[\ell_t, u_t]$}
\ENDFOR
\end{algorithmic}
\caption{PRECiSE-A-CO96: Portfolio REgret for Confidence SEquences with Approximation using \citet{CoverO96}}
\label{alg:confidence_intervals_a}
\end{algorithm}

The following theorem shows that Algorithm~\ref{alg:confidence_intervals_a} computes a valid confidence sequence. The proof is in Section~\ref{sec:proof_precise-a}.
\begin{theorem}
\label{thm:precise-a}
Let $\delta \in (0,1)$. Assume $ X_1, X_2, \dots$ a sequence of random variables such that for each $i$ we have $0 \leq X_i \leq 1$ and
\[
\E[X_i| X_1, \dots, X_{i-1}]=\mu~.
\]
  With the notation in Algorithm~\ref{alg:confidence_intervals_a}, 
  let $G_t(m) = \onec{\hat{\mu}_t \ge m}  G^\ell_{t}(\til\beta^\ell_{t}(m),m)  + \onec{\hat{\mu}_t < m}  G^u_t(\til\beta^u_{t}(m),m)$. Then, with probability at least $1-\delta$, we have
  \begin{align*}
    \PP\del{\max_t ~ t\cdot \max\{ G_t(\mu), \kl(\hat{\mu}_t, m)\} - R_t \ge \ln(1/\dt) } \le \delta~.
  \end{align*}
  Furthermore, for every $t\ge1$, the confidence set for time $t$
  \begin{align*}
    \cbr{m\in \RR: t \cdot \max\{G_t(m), \kl(\hat{\mu}_t, m)\} - R_t \ge \ln(1/\delta) } 
  \end{align*}
  is an interval.
\end{theorem}

We emphasize that such a trick for computational efficiency is possible \textit{thanks to our regret based construction of the confidence sequence.}
Other betting-based confidence bounds, such as~\citet{Waudby-SmithR21}, cannot be turned into a computationally efficient one because they must run the actual betting algorithm, which necessarily spends $\Omega(t)$ time complexity at time step $t$.
In contrast, our method do not require this; we just need to evaluate the maximum wealth, and any lower approximation of it would result in a valid confidence sequence.
Regarding taking the maximum of our lower approximation $G_t^\ell$ with $\kl(\hat\mu_t, m)$, note that this is done without taking union 
Furthermore, in Algorithm~\ref{alg:confidence_intervals_a} one can combine any finite number of lower approximations of the maximum log wealth as long as they are monotonic when split at $m=\hat{\mu}_t$, without inflating the confidence bound.

\def\tmin{{\min}}
\def\tmax{{\max}}
\begin{remark}
  One can further tighten up the KL-divergence-based lower bound on the maximum log wealth using $X_{t,\tmin} := \min\{X_1, \ldots, X_t \}$ and $X_{t,\tmax} := \max\{X_1, \ldots, X_t\}$.
  That is, the first display in the proof of Proposition~\ref{prop:kl} will lead to a tighter bound since $X$ lies in  $[X_{t,\tmin},X_{t,\tmax}]$ rather than $X\in[0,1]$.

\end{remark}

\section{Law of Iterated Logarithm with Portfolio Algorithms}
\label{sec:lil}

In the previous section, we have used algorithm with an optimal regret to derive time-uniform concentration inequalities. However, does small regret implies good performance?
Universal portfolio is optimal up to the constant $\ln\sqrt{2\pi}$ in the log wealth. So, at least if we compare ourselves with the optimal rebalanced portfolio, we are not losing much. However, a minimax regret is not a guarantee of optimal wealth! In fact, there might exists a better algorithm that, for example, has smaller regret against a smaller class of competitors that might include the optimal one for our specific problem. Indeed, this is exactly the strategy used in \citet[Section 7.2]{JunO19} to obtain a time-uniform concentration with a $\log \log t$ dependency rather than a $\log t$. So, here we explain how to derive the law of iterated logarithm with a portfolio algorithm.

The strategy is very simple: we will just use a different mixture, with more mass around the optimal betting fractions when $t$ goes to infinity, that is around 0. Then, thanks to a novel regret guarantee, we will again approximate the log wealth of the algorithm as the log wealth of the constant betting fraction strategy minus the regret and use Theorem~\ref{thm:ville}.

We will use directly the formulation in terms of betting on a continuous coin, that is
\begin{align}\label{eq:wealth-lil}
    \Wealth_t(m)
    = \int_{-1}^{1} \prod_{i=1}^T (1+ \beta (X_i-m)) d F(\beta),    
\end{align}
where $F$ is defined similarly to the mixture used in \citet[Example 3]{Robbins70}:
\[
F(\beta)
= \frac{1}{|\beta| h(\frac{1}{|\beta|})}, \ |\beta|\leq 1~.
\]
In particular, we will use $h(x) = \frac{2 }{\ln \ln c} \ln \frac{c}{x} (\ln \ln \frac{c}{x})^2$ for $c= 6.6 e$.
The choice of $c$ assures that $h(x)$ is decreasing in $[0,1]$. 
We call the resulting algorithm as PRECiSE-R70 (PRECiSE using \citet{Robbins70}), which is described in Algorithm~\ref{alg:confidence_intervals_lil}.

\begin{algorithm}[t]
\begin{algorithmic}[1]
\STATE{\textbf{Input:} $\delta \in (0,1)$, random variables $X_1, X_2 \dots$ in $[0,1]$}
\STATE{$\ell_0=0, u_0=1$}
\FOR{$t=1,2,\dots$}
\STATE{$\hat{\mu}_t=\frac{1}{t} \sum_{i=1}^t X_t$}
\STATE{$\Wealth^*_t(m)=\max_{\beta \in [-1,1]} \ \prod_{i=1}^t (1+\beta (X_i-m))$}
\STATE{$\beta^*_t(m)=\argmax_{\beta \in [-1,1]} \ \prod_{i=1}^t (1+\beta (X_i-m))$}
\STATE{$q_t(m) = \min_{i=1,\dots,t} \ (X_i-m)\sign(\beta^*_t)$}
\STATE{$V_t(m) = \sum_{i=1}^{t}(X_i - m)^2$}
\STATE{$\tilde{\Delta}_t(m):=\begin{cases}\frac{1+\min(q_t(m)\beta^*_t(m),0)}{\sqrt{V_t(m)}}, & |\beta^*_t(m)|<1\\
0, & |\beta^*_t(m)|=1
\end{cases}$}
\STATE{$\Delta_t(m)=|\beta^*_t(m)|-\max(|\beta^*_t(m)|-\tilde{\Delta}_t(m),0)$}
\STATE{$R_t(m) = \min\left( \frac{1-1/\Wealth^*_t(m)}{\ln \Wealth^*_t(m)} |\beta^*_t(m)|, \exp\left(-\frac{\Delta^2_t(m)}{2(1+\min(q_t(m)\beta^*_t(m),0))^2 V_t(m)}\right) \Delta_t(m)\right) F(\beta^*_t(m))$}
\STATE {Compute $m_\ell = \min\left\{m \in [0,\hat{\mu}_t]: \Wealth^*_t(m)\le \frac{R_t(m)}{\delta}\right\}$ with binary search}
\STATE {$\ell_t = \max\cbr{m_\ell, \ell_{t-1}}$}
\STATE {Compute $m_u = \min\left\{m \in [\hat{\mu}_t,1]: \Wealth^*_t(m)\le \frac{R_t(m)}{\delta}\right\}$ with binary search}
\STATE {$u_t = \min\cbr{m_u, u_{t-1}}$}
\STATE {Output $[\ell_t, u_t]$}
\ENDFOR
\end{algorithmic}
\caption{PRECiSE-R70: Portfolio REgret for Confidence SEquences using \citet{Robbins70}}
\label{alg:confidence_intervals_lil}
\end{algorithm}

Note that here we restrict the betting fractions to $[-1,1]$ for simplicity. 
One could consider the entire interval $[-\frac{1}{1-m},\frac{1}{m}]$ by scaling the prior to fit this interval, but it seems that the final rate would depend on $\frac{1}{\mu}$ and $\frac{1}{1-\mu}$ that can be arbitrarily big.
Alternatively, one take a mixture of the prior above (supported on [-1,1]) and the uniform distribution over $[-\frac{1}{1-m},\frac{1}{m}]$ as we discuss later in Remark~\ref{rem:mixing}.

The following theorem shows the guarantee. The proof is in Section~\ref{sec:proof_lil}.
\begin{theorem}
\label{thm:lil}
The regret of the strategy~\eqref{eq:wealth-lil} is bounded by $R_t(m)$ defined in Algorithm~\ref{alg:confidence_intervals_lil}. 
Furthermore, let $\delta \in (0,1)$ and assume $ X_1, X_2, \dots$ to be a sequence of random variables such that for each $i$ we have $0 \leq X_i \leq 1$ almost surely and
\[
\E[X_i| X_1, \dots, X_{i-1}]=\mu~.
\]
Run Algorithm~\ref{alg:confidence_intervals_lil} on $X_1, X_2, \dots$.
Then, with probability at least $1-\delta$ we have that $\mu \in [\ell_t, u_t], \forall t\ge 1$.
Moreover, we have that $u_t-\hat{\mu}_t$ and $\hat{\mu}_t-\ell_t$ are upper bounded by
\[
\frac{1}{t}\sqrt{\max\left(\frac{U_t}{1-2U_t/t},1\right) 2 V_t} + \frac{\frac43 U_t}{t-2U_t}+\frac{24}{t} \ln \frac{7}{6\delta},
\]
where $V_t=\sum_{i=1}^t (X_i-\hat{\mu}_t)^2$, $U_t:= -\fr12 W_{-1} \left(-\frac{2}{\del{\frac{20}{3\delta} \cdot h\left(\frac{1}{2 + \sqrt{V_t/2}}\right)}^2}\right)
=O(\ln\frac{ \ln^3 t}{\delta})$, and $W_{-1}$ is the negative branch of the Lambert function.

Furthermore, $u_t-\hat{\mu}_t$ and $\hat{\mu}_t-\ell_t$ are  upper bounded by $\frac{\sqrt{2 V_t \ln \frac{\ln V_t}{\delta} }}{t}+ o\left(\sqrt{\frac{\ln \ln t}{t}}\right)$ as $t\to \infty$.
\end{theorem}

We want to stress that the regret guarantee could be improved: We aimed for a regret guarantee easy to calculate. However, improvements are possible only in the small sample regime, because the confidence sequences we obtain are asymptotically optimal.

A similar guarantee was first obtained by \citet{Balsubramani14} and a variance-oblivious one is also implied by Ville's inequality and \citet[Theorem 12]{McMahanO14}. However, we want to stress the fact that Theorem~\ref{thm:lil} is only an upper bound to the output of Algorithm~\ref{alg:confidence_intervals_lil}. In practice, the confidence sequences are much smaller. Moreover, as far as we know, this is the first finite-time regret upper bound for a $F$-weighted portfolio algorithm with $F$ equal to the one in \citet{Robbins70}.

\begin{remark}[Mixing PRECiSE-CO96 with PRECiSE-R70]
  \label{rem:mixing}
Any convex combination of betting strategies is still a betting strategy. Hence, we can use PRECiSE-CO96 and PRECiSE-R70 at the same time, allocating, for example, one half of the initial wealth to each of them. The total wealth is just the sum of their wealth. 
In fact, this is equivalent to using a prior that is a mixture of the one of PRECiSE-CO96 and the one of PRECiSE-R70.
In this way, we can guarantee both the LIL property, thanks to PRECiSE-R70, and never-vacuous confidence sequences, thanks to PRECiSE-CO96. Of course, the mixing coefficient becomes a hyperparameter and it could be tuned to trade-off the performance at infinity with the one on few samples. 
\end{remark}

\section{Proof of Our Results}
\label{sec:proofs}

\subsection{Proof of Theorem~\ref{thm:main}}
\label{sec:proof_main}

First of all, we will show how to reduce the problem of betting to a continuous coin to the one of portfolio selection with 2 stocks and solve it with the Universal Portfolio algorithm. Then, we show a sufficient condition on the regret upper bound to guarantee that the confidence sets are intervals.
Finally, we will prove a tighter regret for Universal Portfolio that satisfies the above condition.

\paragraph{From Betting on a Continuous Coin to Universal Portfolio}
Define the regret $\Regret_t$ on the log-wealth of an online betting algorithm receiving continuous coins $c_i \in [-m,1-m]$ where $m\in (0,1)$ as
\begin{equation}
\label{eq:log_regret}
\Regret_t := \max_{\beta \in [-\frac{1}{1-m},\frac{1}{m}]} \ \sum_{i=1}^t \ln(1+\beta c_i) - \sum_{i=1}^t \ln(1+\beta_i c_i)~.
\end{equation}
If we knew an upper bound on $\Regret_t$, then it would be immediate to calculate a lower bound on the log wealth of the algorithm.

So, the only thing we need to know is only the regret of the best possible betting algorithm. It turns out this is trivial! In fact, with the following straightforward lemma, we can reduce this problem to the portfolio selection problem with 2 stocks, where an algorithm with the minimax regret is known.
\begin{lemma}
\label{lemma:reduction}
For a continuous coin $c \in [-m, 1-m]$, set two market gains as $w_1=1+\frac{c}{m}$ and $w_2=1-\frac{c}{1-m}$.
Note that $w_1,w_2\geq0$ so they are valid market gains. 
Define $[b,1-b]$ to be the play of a 2-stocks portfolio algorithm, where $0\leq b\leq1$. 
Then, by taking 
\[
\beta = -\frac{1}{1-m} + \left(\frac{1}{1-m}+\frac{1}{m}\right)b \in \left[-\frac{1}{1-m},\frac{1}{m}\right]
\]
as the signed betting fraction, a continuous-coin-betting algorithm on $c$ ensures that the gain in the coin betting problem coincides with the gain in the portfolio selection problem.
\end{lemma}
\begin{proof}
  After $b$ is revealed, the wealth is updated by multiplying the previous wealth by the following:
\begin{align*}
w_1 b + w_2 (1-b)
&= b + b \frac{c}{m} + \left(1-\frac{c}{1-m}\right)(1-b)
= b + b \frac{c}{m} + 1 - b -\frac{c}{1-m} +\frac{c}{1-m} b \\
&= 1 + c \left(-\frac{1}{1-m}+\left(\frac{1}{1-m}+\frac{1}{m}\right)b\right)
= 1+ c \beta,
\end{align*}
where $\beta$ is the signed betting fraction equal to $-\frac{1}{1-m}+(\frac{1}{1-m}+\frac{1}{m})b$.
Given that $b \in [0,1]$, the range of the betting fractions is in $[-\frac{1}{1-m},\frac{1}{m}]$ as we wanted. 
\end{proof}

We can also write $b$ in terms of $\beta$ by $b = (1-m)m \beta + m$.
Note that $\beta=-\frac{1}{1-m}$ corresponds to $b=0$, $\beta=0$ corresponds to $b=m$, and $\beta=\frac{1}{m}$ corresponds to $b=1$. We stress that we do not need to use this transformation in the algorithm.
Instead, it is enough to know that \emph{it exists}.

This reduction allows us to use any portofolio selection algorithm to bet on an asymmetric continuous coin.
Next, we prove a sufficient condition on the regret upper bound to obtain intervals, and then we show a tight empirical upper bound on the regret of Dirichelet(1/2,1/2)-weighted portfolio.

\paragraph{Sufficient Condition on the Regret to Obtain Intervals}

Here, we prove a sufficient condition on the regret upper bound to guarantee that the confidences sequences are intervals.
In particular, the following Lemma implies that any regret upper bound that only depends on $t$ will give rise to a quasiconvex function, which means that the confidence sequences computed by~\eqref{eq:to_be_inverted} are intervals when $R_t$ is independent of $m$.

\begin{lemma}
\label{lemma:intervals}
Let $X_i \in [0,1]$ for $i=1, \dots t$. Define $G(\beta,m):=\sum_{i=1}^t \ln(1+\beta (X_i-m))$ and $H(m):=\max_{\beta \in [-\frac{1}{1-m},\frac{1}{m}]}\ G(\beta,m)$. Define  $\hat{\mu}_t:=\frac{1}{t}\sum_{i=1}^t X_i$. Then, $H(m)$ in nonincreasing in $[0,\hat{\mu}_t)$ and nondecreasing in $(\hat{\mu}_t,1]$. Hence, $H(m)$ is quasiconvex in $[0,1]$.
\end{lemma}
\begin{proof}
Denote by $\beta^*(m)=\argmax_{\beta \in [-\frac{1}{1-m},\frac{1}{m}]}\ G(\beta,m)$.

The derivative of $G$ w.r.t its first argument is
\[
G'(\beta,m)= \sum_{i=1}^t \frac{X_i-m}{1+\beta (X_i -m)}~.
\]

Then, we have that $G'(0,\hat{\mu}_t)=0$. Given that $G(\beta,m)$ is concave in $\beta$, then $G(\beta,\hat{\mu}_t)$ has a maximum w.r.t. the first argument in $\beta=0$ and the value of the function is 0. 

For $m'>\hat{\mu}_t$, $G'(0,m')<0$. 
Since $G(\beta,m')$ is concave in $\beta$, we have $\beta^*(m')<0$. 
In the same way, for $m'<\hat{\mu}_t$ we have $\beta^*(m')>0$. 

Let's start with $m'>\hat{\mu}_t$ and prove that $H$ is nondecreasing, the other side is analogous. Consider $m_1>m_2>\hat{\mu}_t$.
Given that $\beta^*(m_2)<0$, we have 
\[
H(m_2)
=G(\beta^*(m_2),m_2)
\leq G(\beta^*(m_2),m_1)
\leq G(\beta^*(m_1),m_1) = H(m_1),
\]
where the first inequality is due to the fact that $G(\beta,m)$ is nondecreasing in $m$ when $\beta<0$ and the second inequality is due to the fact that the negative part of the interval $[-\frac{1}{1-m_1},\frac{1}{m_1}]$ contains the negative part of the interval $[-\frac{1}{1-m_2},\frac{1}{m_2}]$ and we know the maximum $\beta$ is negative. 
\end{proof}

\paragraph{Data-dependent Regret of Universal Portfolio Selection with 2 Stocks}
Here, we introduce a data-dependent regret upper bound for the universal portfolio algorithm.

\citet{CoverO96} proved that setting the mixture distribution $F$ equal to the Dirichelet(1/2,1/2) distribution gives an upper bound to the regret of
\[
\Regret_T 
\leq \ln\frac{\sqrt{\pi}\Gamma(t+1)}{\Gamma(t + \frac 1 2)},
\]
that is optimal up to constant additive terms.
As we anticipated above, this regret is tight in the case that in each round the market gains are exactly one 0 and one 1. However, for other sequences the regret can be smaller.

So, here we derive an easy-to-calculate data-dependent regret guarantee. The following upper bound is essentially in the proofs in  \citet{CoverO96}.
\begin{theorem}
\label{thm:data_dep_regret}
Denote by $[b^*_t, 1-b^*_t] = \argmax_{\bb \in B} \ \Wealth_t(\bb)$. Then, the regret of the Dirichelet(1/2,1/2)-weighted portofolio algorithm satisfies
\[
\Regret_t 
\leq \max_{0 \leq k\leq t} \ f(b^*_t,k,t),
\]
where $f(b,k,t)$ is defined in \eqref{eq:f}.
\end{theorem}
\begin{proof}
From \citet[Lemma~2]{CoverO96}, in the 2 stocks case we have
\[
\frac{\max_{\bb \in B} \ \Wealth_t(\bb)}{\Wealth_t}
\leq \max_{0 \leq k \leq t} \ \frac{(b^*_t)^k (1-b^*_t)^{t-k}}{\int_{0}^1 b^k (1-b)^{t-k} \, d F(b)}~.
\]
Moreover, from equation (64) in \citet{CoverO96}, we have
\[
\int_{0}^1 b^k (1-b)^{t-k} \, d F(b)
= \frac{\Gamma(k+\fr12)\Gamma(t-k+\fr12)}{(\Gamma(\fr12))^2\Gamma(t+1)}~.
\]
Putting all together, we have the stated bound.
\end{proof}

Unfortunately, it can be verified numerically that the bound in Theorem~\ref{thm:data_dep_regret} does not give rise to an interval.
The reason is that $b^*_t$ depends on $m$ and the regret gets smaller when $b^*_t$ is close to 1/2, but at the same time $\Wealth(b^*_t)$ also decreases in this case. Hence, $\Wealth(b^*_t)$ minus the regret upper bound is not guaranteed to be quasi-convex in $m$ and indeed we verified numerically that we can get non-quasi-convex functions.
 
So, we propose the following variation.
For each $t$ we have to check $m$ in a range equal to the confidence set calculated at $t-1$. Assume by induction that the confidence set with $t-1$ samples is an interval $[\ell_{t-1},u_{t-1}]$, we now construct a new interval. The key idea is the following one: as we vary $m$ in $[\ell_{t-1}, \hat{\mu}_t)$, $\beta^*_t$ moves monotonically from a positive value to 0.
We now calculate the worst value of the regret for all the $\beta^*_t$ in this range. This is now a value independent of $m$ and so Lemma~\ref{lemma:intervals} applies.
The same reasoning holds for the interval $m \in (\hat{\mu}_t,u_{t-1}]$.

The following Lemma characterize the local maxima of the function appearing in the regret. 
To facilitate the understanding of the proof, we plot the behavior of the key function $h(b)$ defined in the following lemma for $t=10$ in Figure~\ref{fig:fig_lemma}.

\begin{figure}
    \centering
    \includegraphics[width=0.8\textwidth]{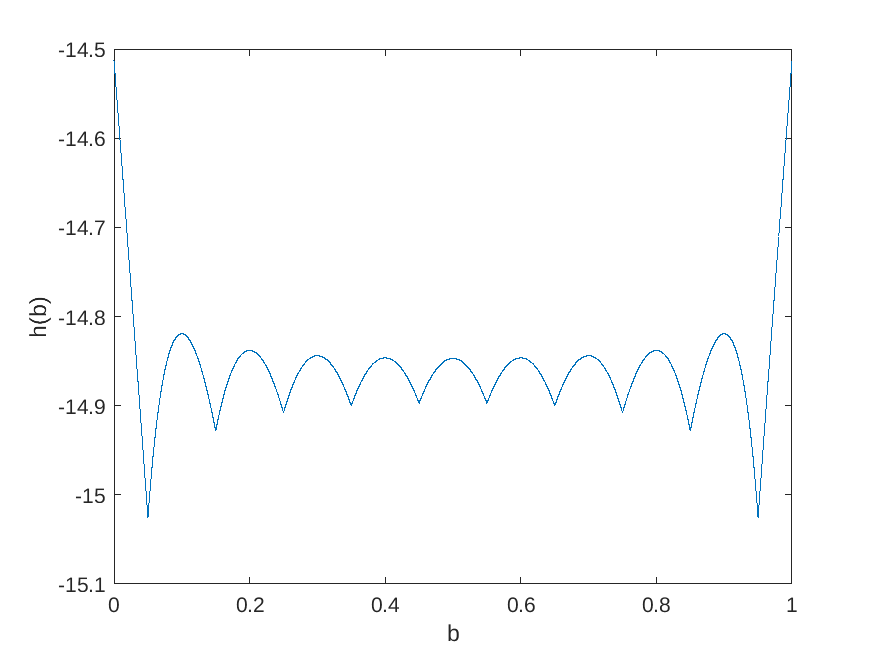}
    \caption{Function $h(b)$ for $t=10$ used in Lemma~\ref{lemma:local_maxima}.}
    \label{fig:fig_lemma}
\end{figure}

\begin{lemma}
\label{lemma:local_maxima}
Let $t\geq 1$ integer and denote by $f(b,k,t)=\ln \frac{b^k (1-b)^{t-k} \Gamma(t+1)(\Gamma(1/2))^2}{\Gamma(k+1/2) \Gamma(t-k+1/2)}$.
Define $h(b) =\max_{k=0, \dots, t} f(b,k,t)$.
Then, for any $0\leq b\leq 1$, the following hold
\begin{itemize}
\item $\{\lceil t b -0.5\rceil, \lfloor t b+0.5\rfloor\}= \argmax_{k=0, \dots, t} f(b,k,t)$.
\item The local maxima of the function $h(b)$  are exactly at $\{i/t\}_{i=0}^t$.
\item $h(i/t)$ is nonincreasing for $i=0,\dots, \lceil (t+1)/2 \rceil-1$ and nondecreasing for $i=t-(\lceil (t+1)/2 \rceil-1), \dots,t$.
\end{itemize}
\end{lemma}
\begin{proof}
Fix $b\in[0,1]$ and let  $k^* \in \argmax_{k=0, \dots, t} f(b,k,t)$.
This implies that $f(b,k^*+1,t)$ and $f(b,k^*-1,t)$ are not larger than $f(b,k^*,t)$.
For convenience, let us work with $g(b,k,t) := \exp(f(b,k,t))$.
Then,
\[
g(b,k^*+1,t)/g(b,k^*,t)=
\frac{b(t-k^*-1/2)}{(k^*+1/2)(1-b)}~.
\]
This ratio is less than 1 iff
\[
b(t-k^*-1/2)
\leq (k^*+1/2)(1-b),
\]
that is
\[
k^*\geq b(t-1/2)-1/2(1-b)
= b t -1/2~.
\]
Moreover,
\[
g(b,k^*-1,t)/g(b,k^*,t)
= \frac{(k^*-1/2)(1-b)}{b(t-k^*+1/2)}\leq 1,
\]
that is equivalent to
\[
k^* \leq b(t+1/2)+1/2(1-b)
= b t + 1/2~.
\]
Putting the two inequality together with the constraint that $k^*$ is integer gives the stated expression.
Moreover, the conditions above imply that the arg max becomes only one element in $\{0,\ldots,t\}$ in general, but it could be one of the two elements when $tb/2$ is an integer.
In such edge case, one can see that both achieve the same objective value.
This concludes the first statement of the lemma.

For the second property, for every $k\in\{0,\dots,t\}$, let $I_k$ be the set of $b$'s for which $k \in \argmax_{k=0, \dots, t} f(b,k,t)$.
Then, $\{I_k\}$ forms a partition of $[0,1]$ where they share the boundary with adjacent intervals.
That is, we have $\cup_{k=0}^t I_k =[0,1]$, $I_k \cap I_{k+1} = (k+0.5)/t, \forall k\in\{0,\ldots,t\}$, and $I_k \cap I_{j}=\{\}$ for all $j,k\in\{0,\ldots,t\}$ with $|j-k|>1$. 
On each partition $I_k$, it is easy to see that $h(b)$ is concave and its maximum is $k/t$.
Hence, the local maxima of $h(b)$ are exactly at $\{k/t\}_{k=0}^t$.

The third property directly follows from the proof of Lemma 4 in \citet{CoverO96}.
\end{proof}

We can now present the upper bounds on the regret.
\begin{lemma}
\label{lemma:alg_u_l_bounds}
For any $X_1,\dots,X_t$ in $[0,1]$, define
$\hat{\mu}_t=\frac{1}{t}\sum_{i=1}^t X_t$, and
\[
\phi(x)=\argmax_{b \in [0,1]} \ \sum_{i=1}^t \ln\left( b\left(1+\frac{X_i-x}{x}\right) + (1-b) \left(1-\frac{X_i-x}{1-x}\right) \right)~.
\]
Consider running universal portfolio with 2 stocks, Dirichlet($1/2,1/2$) mixture, and the market gains $\{(1+\frac{X_i-m}{m}, 1-\frac{X_i-m}{1-m})\}_{i=1}^t$.
Then, for any $\ell \in [0, \hat{\mu}_t)$ and $m \in [\ell, \hat{\mu}_t]$, we have that its regret is upper bounded by
\[
\max\left\{f\left(\frac{\lceil \phi(\ell)\cdot t-0.5\rceil}{t},\lceil \phi(\ell)\cdot t-0.5\rceil,t\right),f\left(\frac{\lfloor \hat{\mu}_t\cdot t+0.5\rfloor}{t},\lfloor \hat{\mu}_t\cdot t+0.5\rfloor,t\right)\right\}~.
\]
On the other hand,
for any $u \in(\hat{\mu}_t,1]$ and $m \in [\hat{\mu}_t,u]$, we have that the regret is upper bounded by
\[
\max\left\{f\left(\frac{\lfloor \phi(u)\cdot t+0.5\rfloor}{t},\lfloor \phi(u)\cdot t+0.5\rfloor,t\right),f\left(\frac{\lceil \hat{\mu}_t\cdot t-0.5\rceil}{t},\lceil \hat{\mu}_t\cdot t-0.5\rceil,t\right)\right\}~.
\]
\end{lemma}
\begin{proof}
We will only prove the first upper bound because the other one is analogous.

Observe that when $m$ increases, the market gains of the first stock decrease and the ones of the second stock increases. This means that $\phi(x)$ is nonincreasing. Hence, for $m \in [\ell,\hat{\mu}_t]$ we have that $\phi(m) \in [\phi(\hat{\mu}_t), \phi(\ell)]$. 
Recall that $G(\beta,\hat{\mu}_t)$ is maximized at $\beta=0$ by Lemma~\ref{lemma:intervals}.
Using our reduction in Lemma~\ref{lemma:reduction} and noting that $ \sum_{i=1}^t \ln\left( b\left(1+\frac{X_i-x}{x}\right) + (1-b) \left(1-\frac{X_i-x}{1-x}\right) \right) = G(\beta,x)$, one can show that $\phi(\hat{\mu}_t) = \hat{\mu}_t$.
Using the properties of the local maxima of the upper bound of universal portfolio in Lemma~\ref{lemma:local_maxima}, we have the stated bound.
\end{proof}

We can finally prove Theorem~\ref{thm:main}.
\begin{proof}[Proof of Theorem~\ref{thm:main}]
It is enough to use Lemma~\ref{lemma:alg_u_l_bounds} to derive a lower bound to the wealth and use Theorem~\ref{thm:ville}.
\end{proof}

\subsection{Proof of Theorem~\ref{thm:algo}}
\label{sec:proof_thm_algo}

\begin{proof}
The algorithm simply numerically invert the inequality in \eqref{eq:to_be_inverted}, using two different upper bounds to the regret for the lower and upper values of the confidence intervals, in $R_t^\ell$ and $R_t^u$ respectively. The particular expression of the regret is given in Lemma~\ref{lemma:alg_u_l_bounds}.
Note that lines 8 and 12 take care of the maximum over $t$ in \eqref{eq:to_be_inverted}. Moreover, Lemma~\ref{lemma:intervals} proves that $L_t$ and $U_t$ are intervals, so we can use a binary search procedure.
\end{proof}

From the above proof, it should be clear that we can refine the regret upper bounds at each step of the binary search procedure. However, numerically the advantage is tiny, so for simplicity we decided not to present this variant.

\subsection{Proof of Theorem~\ref{thm:never_vacuous}}
\label{sec:proof_thm_never_vacuous}

\begin{proof}
First of all, it should be clear that the width of the confidence intervals $u_t-\ell_t$ are always smaller than the one calculated with 1 sample, $u_1-\ell_1$.
With only one sample, the upper and lower bound have a closed formula.
Indeed, the argmax of \eqref{eq:to_be_inverted} is achieved in $\beta=\frac{1}{m}$ if $X_1-m>0$ and in $\beta=-\frac{1}{1-m}$ for $X_1-m<0$. This implies that
\[
\ell_1 = \frac{X_1}{\exp(C_1+ \ln \tfrac{1}{\delta})},
\]
and 
\[
u_1 = 1-\frac{1-X_1}{\exp(C_1+ \ln \tfrac{1}{\delta})}~.
\]
Moreover, from direct calculation we also get that $C_1$ is equal to $\ln 2$.
Subtracting the upper bound from the lower bound, we get the stated bound for any $X_1$.
\end{proof}

\subsection{Proof of Theorem~\ref{thm:approx_inversion}}
\label{sec:proof_thm_approx}

First, we state a technical lemma.
\begin{lemma}
  \label{lemma:max_approx_wealth}
  Let $f(x)=A x + B ( \ln(1-|x|)+|x|)$, where $A \in \R$ and $B\geq 0$. Then, $\argmax_{x \in [-1,1]} \ f(x) = \frac{A}{|A|+B}$ and $\max_{x \in [-1,1]} \ f(x) = B\psi(\frac{A}{B}) \geq \frac{A^2}{4/3|A|+2B}$, where $\psi(x)=|x|-\ln(|x|+1)$.
\end{lemma}
\begin{proof}
  If $B=0$, we have that argmax is $\sign(A)$. If $A=0$, the argmax is 0.
  Hence, in the following we can assume $A$ and $B$ different than 0.
  
  We can rewrite the maximization problem as
  \begin{align*}
    \argmax_{x} \ f(x)
    = B \argmax_{x} \ \frac{A}{B}x + \ln(1-|x|)+|x|~.    
  \end{align*}
  From the optimality condition, we have that
  $\frac{A}{B} - \frac{\sign(x^*)}{1-|x^*|}+\sign(x^*) = 0$, 
  that implies $x^*=\frac{A}{|A|+B}$.
  Substituting this expression in $f$, we obtain the stated expression. 
  The inequality is obtained by the elementary inequality $\ln(1+x) \le x\cdot\frac{6+x}{6+4x}$ for $x\ge 0$.
\end{proof}

We can now prove Theorem~\ref{thm:approx_inversion}.
\begin{proof}
  For a given $t$, set $\epsilon_t$ equal to $\mu-\hat{\mu}_t$, so that $\epsilon+\hat{\mu}_t \in [0,1]$.
  Symmetrizing \citet[equation 4.12]{FanGL15}, we have for any $|x|\leq 1$ and $|\beta|\leq 1$
  \begin{equation}
    \ln(1+\beta x) \geq \beta x + (\ln(1-|\beta|)+|\beta|)x^2~. 
    \label{eq:fan}
  \end{equation}
  Hence, for any $\beta \in [-1,1]$, we have
  \begin{align*}
    \sum_{i=1}^t \ln(1+\beta (X_i-\mu) )
    &= \sum_{i=1}^t \ln(1+\beta (X_i-\hat{\mu}_t-\epsilon_t) ) \\
    &\geq \beta \sum_{i=1}^t (X_i-\hat{\mu}_t -\epsilon_t) + (\ln(1-|\beta|)+|\beta|) \left(\sum_{i=1}^t (X_i-\hat{\mu}_t)^2 + \epsilon_t^2 t - 2 \epsilon_t \sum_{i=1}^t (X_i-\hat{\mu}_t)\right) \\
    &= -\epsilon_t \beta t + (\ln(1-|\beta|)+|\beta|) \left(\sum_{i=1}^t (X_i-\hat{\mu}_t)^2 +  \epsilon_t^2 t\right) ~.
  \end{align*}
  Hence, we have
  \begin{align*}
    \max_{\beta \in [-1,1]} \  \sum_{i=1}^t \ln(1+\beta (X_i-\hat{\mu}_t-\epsilon_t)) 
    = \left(\sum_{i=1}^t(X_i-\hat{\mu}_t)^2 +  \epsilon_t^2 t\right) \psi\left(\frac{|\epsilon_t| t}{\sum_{i=1}^t(X_i-\hat{\mu}_t)^2 +  \epsilon_t^2 t }\right),
  \end{align*}
  where $\psi(x) = |x| - \ln(|x|+1)$ and the equality is due to Lemma~\ref{lemma:max_approx_wealth}. From the inequality in Lemma~\ref{lemma:max_approx_wealth} we also obtain
  \[
  \max_{\beta \in [-1,1]}\  \sum_{i=1}^t \ln(1+\beta (X_i-\hat{\mu}_t-\epsilon_t)) 
  \geq \frac{\epsilon_t^2 t^2}{4/3 |\epsilon_t| t + 2 \sum_{i=1}^t(X_i-\hat{\mu}_t)^2 +  2 \epsilon_t^2 t}~.
  \]
  
  Now, note that for any $\mu \in [0,1]$ the interval $[-\frac{1}{1-\mu},\frac{1}{\mu}]$ is contained in $[-1,1]$.
  Hence, from Theorem~\ref{thm:main}, uniformly on all $t$ with probability at least $1-\delta$, we have
  \begin{align*}
    \frac{\epsilon_t^2 t^2}{4/3 |\epsilon_t| t + 2 \sum_{i=1}^t(X_i-\hat{\mu}_t)^2 +  2 \epsilon_t^2 t}
    \leq \Regret_t + \ln \frac{1}{\delta}~.
  \end{align*}
  Assuming $\epsilon_t$ positive and solving for it, we have the stated upper bound.
  The expression for negative $\epsilon_t$ has the signed flipped, by the symmetry of the expression.
\end{proof}

\subsection{Proof of Theorem~\ref{thm:precise-a}}
\label{sec:proof_precise-a}

\begin{proof}
  We focus on showing that
    \begin{align*}
    \PP\del{\max_t \ t\cdot G_t(\mu) - R_t \ge \ln(1/\dt) } \le \dt
  \end{align*}
  and that, for every $t\ge1$, the confidence set for time $t$
  \begin{align*}
    \cbr{m\in \RR: t \cdot G_t(m) - R_t \ge \ln(1/\dt) } 
  \end{align*}
  is an interval.
  Modifying the proof below for the theorem statement is a trivial modification using the fact that $\kl(\hat{\mu}_t, \mu)$ lower bounds the maximum wealth due to Proposition~\ref{prop:kl}, $\kl(\hat{X}_t, x)$ is piece-wise monotonic when split at $x=\hat{\mu}_t$, and the maximum of two nonincreasing (nondecreasing) function is nonincreasing (nondecreasing) respectively.
  
  First, we show a tight lower bound on the log wealth that depends on the sign of $\beta$.
  Define $c_i := X_i - m$.
  Throughout the proof, we drop the subscript $t$ from $\barV_t$, $G^\ell_t$, $G^u_t$, etc. to reduce clutter.

  If $\beta \in [0, 1/m]$, using \citet[Eq. 4.11]{FanGL15} with $\lam=\beta m \in [0,1]$ and $\xi = c_i/m \ge -1$, we have
  \begin{align*}
    \ln(1 + \beta c_i) 
    &\ge \beta c_i - \beta^2 c_i^2 \cd \fr{-\ln(1-\beta m) - \beta m}{\beta^2 m^2} 
  \\&= \beta c_i - c_i^2 \cd \fr{-\ln(1-\beta m) - \beta m}{m^2}, 
  \end{align*}
  where we take $\ln(0)$ as $-\infty$.
  Recall $\barV = \frac1t \sum_{i=1}^t (X_i - \hat{\mu}_t)^2$.
  Then, with algebra, one can show that
  \begin{align*}
    \max_{\beta\in[0,1/m]} \ \frac1t \sum_{i=1}^t \ln(1+\beta c_i) 
    \ge 
    \max_{\beta\in[0,1/m]} \ \sbr{\beta(\hat{\mu}_t - m) - (-\ln(1-\beta m) -\beta m)\fr{(\barV + (\hat{\mu}_t - m)^2)}{m^2}
    = G^\ell(\beta, m)}
  \end{align*}
  A simple algebra tells us that the maximum is achieved at $\til\beta_\ell(m)$ (defined in Algorithm~\ref{alg:confidence_intervals_a}).

  If $\beta \in [-1/(1-m),0]$, we can employ a similar argument to derive
  \begin{align*}
    &\max_{\beta\in[-1/(1-m),0]} \ \frac1t \sum_i \ln(1+\beta c_i) 
    \\&\ge \max_{\beta\in[-1/(1-m),0]} \sbr{-\beta(m-\hat{\mu}_t) - (-\ln(1+(1-m)\beta) + (1-m)\beta) \fr{\barV+(m-\hat{\mu}_t)^2}{(1-m)^2} = G^u(\beta,m) }~.
  \end{align*}
  The maximum is achieved at $\til\beta_u(m)$ (defined in Algorithm~\ref{alg:confidence_intervals_a}).
  
  Let $\beta^*(m) = \max_{\beta\in[-1/(1-m),1/m]} \ \sum_{i=1}^t \ln( 1 + \beta(X - m)) $.
  Examining the gradient of the objective function at $\beta=0$, one can see that $\hat{\mu}_t \ge \mu$ implies that $\beta^*(\mu) \ge 0$ and $\hat{\mu}_t \le \mu$ implies that $\beta^*(\mu) \le 0$.

  If $\hat{\mu}_t \ge \mu$, then using $\beta^*(\mu) \ge 0$ we have
  \begin{align*}
    \sum_i \ln(1 + \beta^*(\mu)\cdot (X_i - \mu)) 
    = \max_{\beta\in[0,1/\mu]} \ \sum_i \ln(1 + \beta c_i) 
    \ge \max_{\beta\in[0,1/\mu]} \ G^\ell(\beta, \mu)~.
  \end{align*}
  Otherwise, we have
  \begin{align*}
    \sum_i \ln(1 + \beta^*(\mu)\cdot (X_i - \mu)) 
    \ge \max_{\beta\in[-1/(1-\mu),0]} \ G^u(\beta, \mu)~.
  \end{align*}
  
  Hence, from Theorem~\ref{thm:main}, with probability at least $1-\dt$, we have, for all $t$ 
  \begin{align*}
    G_t(\mu) \le \frac1t \ln(e^{\Regret_t}/\delta) \le  \frac1t \ln(e^{R_t}/\delta) ~.
  \end{align*}

  It remains the show that the confidence set stated in the theorem forms an interval.
  In this proof, we focus on showing that the lower side only (i.e., the confidence set intersecting with  $[0,\hat{\mu}_t]$) as the proof for the upper side is symmetric.
  
  Let $\beta(m) = \arg \max_{\beta \in [0,1/m]} \ G^\ell(\beta,m)$ and $F(m) = \max_{\beta \in [0,1/m]} \  G^\ell(\beta,m)$.
  It suffices to show that $F$ is nonincreasing in $[0,\hat{\mu}_t]$.
  That is, we claim that if $m_1 < m_2$ then $F(m_1) = G^\ell(\beta(m_1), m_1) \ge G^\ell(\beta(m_2), m_2) = F(m_2)$.
  
  To prove the claim, we will show below that $G^\ell(\beta(m_2),m_1) \ge G^\ell(\beta(m_2), m_2)$.
  If this is true, then we have
  \begin{align*}
    F(m_1) = G^\ell(\beta(m_1), m_1) \ge G^\ell(\beta(m_2),m_1) \ge G^\ell(\beta(m_2), m_2) = F(m_2),
  \end{align*}
  where the first inequality is due to the definition of $\beta(m_1)$ and the fact that $[0,1/m_1] \supseteq [0,1/m_2]$.
  
  Now, let us show that $G^\ell(\beta(m_2),m_1) \ge G^\ell(\beta(m_2), m_2)$.
  To see this, it suffices to show that $G^\ell(\beta, m)$ is nonincreasing in $m \in [0, m_2]$ where $\beta \in [0,1/m_2]$.
    \begin{align*}
      \fr{\der}{\der m} G^\ell(\beta,m) = \cdots = \fr{1}{m^3}  \beta m \del{-m^2 - \underbrace{\del{\fr{1}{1-\beta m} + 1 + \fr{2\ln(1-\beta m) }{\beta m} }}_{\tsty =: A}\cd \del{\barV + (\hat{\mu}_t - m)^2} }~.
    \end{align*}
    It suffices to show that $A \ge 0$.
    Using the fact that $q:=\beta m \in [0,1]$, we have
    \begin{align*}
      A = \frac{1}{1-q}  + 1 + \fr{2\ln(1-q)}{q}~.
    \end{align*}
    Let $y = 1-q$.
    Using the elementary inequality $\ln(1-q) = \ln(y) \ge -\fr{1-y}{\sqrt{y}} $, $\forall y \in [0,1]$, we have that $\ln(y) \ge -\fr{1-y}{\sqrt{y}}  \ge  -\fr{1-y}{y}\sqrt{y} \ge -\fr{1-y}{y}\cd \fr{1 + y}{2}  = - \fr{1-y^2}{2y} = - \fr{1-(1-y)^2}{2(1-q)} $.
    With this bound, we have
    \[
      A \ge \fr{1}{1-q}  + 1 - \fr{2-q}{1-q} = 0~. \qedhere
    \]
\end{proof}

\subsection{Proof of Theorem~\ref{thm:lil}}
\label{sec:proof_lil}

First, we state two technical lemmas.

\begin{lemma}
\label{lemma:beta-invert} 
  Assume $\beta \in [0,1)$, and $\beta \le \frac{1}{A}(1 + \beta)^2$.
  If $A \ge 4$, then $\beta \le \fr{2}{A - 2 + \sqrt{(A - 2)^2 - 4 }  }$.
  Furthermore, if $A \ge 5$,  
  $\beta \le \fr1A + \fr{5}{A^2}$.   
\end{lemma}
\begin{proof}
  Solving the quadratic equation, we have an upper bound on $\beta$ and a lower bound on $\beta$.
  One of the two is true.
  However, using $A\ge4$, one can show that the lower bound implies $\beta \ge 1$.
  Thus, one can take the upper bound.
  
  For the second statement, one can show that
  \begin{align*}
    \fr{2}{A-2 + \sqrt{(A - 2)^2 - 4}} = \fr1A + \fr{A^2 + 2A - \sqrt{A^2\cd A(A-4)}  }{A - 2 + \sqrt{A(A-4)} }  \cd \fr{1}{A^2}  
  \end{align*}
  Note that $\fr{A^2 + 2A - \sqrt{A^2\cd A(A-4)}  }{A - 2 + \sqrt{A(A-4)} }$ is nonincreasing in $A$.
  Using the bound $A \ge 5$ concludes the proof.

\end{proof}

\begin{lemma}
\label{lemma:lambert_log_ineq}
Let $A>0$, $\ln(A)+\frac{B}{A}\geq 1$, and $x\leq A \ln(x)+B$. Then,
\[
x
\leq -A \cdot W_{-1}\left(-\frac{1}{A}\exp\left(-\frac{B}{A}\right)\right)
\leq A \ln A + B + A \sqrt{2\left(\ln A +\frac{B}{A}-1\right)},
\]
where $W_{-1}$ is the negative branch of the Lambert function and the first inequality is tight.
\end{lemma}
\begin{proof}
To obtain the first inequality we observe that
\begin{align*}
x \leq A \ln x+B
\Leftrightarrow \frac{x}{A} \leq \ln\left( \frac{x}{A}\right)+\ln A+\frac{B}{A}~.
\end{align*}
Hence, we can solve the associated equality directly obtaining that the biggest solution is given by $\frac{x}{A}=W_{-1}\left(-\frac{1}{A}\exp(-\frac{B}{A})\right)$, that gives also the solution of the inequality. The second one is from the lower bound on the Lambert function in  \citet{Chatzigeorgiou13}.
\end{proof}

We can now prove the theorem.
\begin{proof}[Proof of Theorem~\ref{thm:lil}]
For brevity of notation, we drop the dependency on $m$ on the quantities defined in the algorithm, i.e., $V_t:=V_t(m)$. Also, for a fixed $m$, define $\theta_t := \sum_{i=1}^t (X_i-m)$ and $\Wealth_t(\beta):=\prod_{i=1}^t (1+\beta (X_i-m))$.
From the above it is easy to verify that $V_t = V_t(\hat{\mu}_t) + \frac1t \theta_t^2$.

Note that from the log-concavity of the wealth function, we have for any $\beta_2>\beta_1\geq0$ that
\begin{align}
\int_{\beta_1}^{\beta_2} \Wealth_t(\beta) F(\beta) \dif \beta 
&\geq F(\beta_2) \int_{\beta_1}^{\beta_2} \Wealth_t(\beta) \dif \beta
= (\beta_2-\beta_1) F(\beta_2) \int_{0}^{1} \Wealth_t(\beta_1(1-a)+a\beta_2) \dif \beta \nonumber \\
&\geq (\beta_2-\beta_1) F(\beta_2) \int_{0}^{1} \Wealth^{1-a}_t(\beta_1) \Wealth^a_t(\beta_2) \dif \beta \nonumber \\
&= F(\beta_2) (\beta_2 - \beta_1) \frac{\Wealth_t(\beta_2) - \Wealth_t(\beta_1)}{\ln\frac{\Wealth_t(\beta_2) }{\Wealth_t(\beta_1)} }~. \label{eq:wealth_log_concavity}
\end{align}

Our first term in the regret is obtained using the fact that $\frac{x-1}{\ln x}\geq \sqrt{x}$, obtaining
\begin{align}
\Wealth_t 
&\geq F(\beta^*_t) |\beta^*_t| \frac{\Wealth_t(\beta^*_t) - \Wealth_t(0)}{\ln\frac{\Wealth_t(\beta^*_t) }{\Wealth_t(0)} }
= F(\beta^*_t) |\beta^*_t| \frac{\Wealth_t(\beta^*_t) - 1}{\ln \Wealth_t(\beta^*_t) } \nonumber \\
&\geq \sqrt{\Wealth_t(\beta^*_t)} |\beta_t^*| F(\beta^*_t)~. \label{eq:regret_lil_1}
\end{align}

For the second term in the regret, we use a Taylor expansion of the log wealth. Denote by $f(\beta)=\ln \Wealth_t(\beta)$.
Hence, when $|\beta^*_t|<1$, for some $\beta$ between $\beta^*_t$ and $\beta_t^*-\Delta_t\sign(\beta^*_t)$, we have
\[
f(\beta^*_t-\Delta_t \sign(\beta^*_t))
= f(\beta_t^*) -\Delta_t \sign(\beta^*_t) f'(\beta^*_t) +\frac{\Delta_t^2}{2}  f''(\beta)~.
\]
From the definition of $q_t$ and $\Delta_t$ in the algorithm, for any $\beta$ between $\beta^*_t$ and $\beta^*_t-\Delta_t \sign(\beta^*_t)$, we have 
\[
f''(\beta)=-\sum_{i=1}^t \frac{(X_i-m)^2}{(1+\beta(X_i-m))^2}
\geq -\sum_{i=1}^t \frac{(X_i-m)^2}{(1+\beta q_t)^2}~.
\]
Hence, for $|\beta_t^*|<1$ we have
\[
f(\beta^*_t-\Delta_t \sign(\beta_t^*) )
\geq f(\beta_t^*)  +\frac{\Delta_t^2}{2}  f''(\beta)
\geq f(\beta_t^*) - \frac{\Delta_t^2 V_t}{2 (1+ \min(\beta^*_t q_t,0))^2},
\]
where the presence of the min is necessary to use $\beta^*_t$ in the denominator.
Using this expression in the integral of the wealth, we have
\begin{equation}
\int_{-1}^{1} \Wealth_t(\beta) F(\beta) \dif \beta
\geq \Wealth_t(\beta^*_t) \Delta_t \exp\left(-\frac{\Delta_t^2 V_t}{2 (1+ \min(\beta^*_t q_t,0))^2}\right) F(\beta^*_t)~. \label{eq:regret_lil_2}
\end{equation}
Putting together \eqref{eq:regret_lil_1} and \eqref{eq:regret_lil_2}, we have the expression of the regret of the portfolio algorithm (i.e., $R_t(m)$ in the algorithm).

Now, we turn our attention to the expression of the confidence sequences.
In the following, we safely assume that $|\theta_t| > \sqrt{2V_t}$ since otherwise we obtain the desired bound.

First, we need to study $\beta^*_t$. If $|\beta^*_t| < 1$, then using Taylor's remainder theorem, we have $f'(0) + \beta_t^* f''(\beta)= f'(\beta_t^*) $, where $\beta$ is between $0$ and $\beta^*_t$. This implies that
\begin{align*}
    \sum_{i=1}^t (X_i-m) - \beta^*_t \sum_{i=1}^t \frac{(X_i-m)^2}{(1 + \beta (X_i-m))^2} = 0, \text{ for some } \beta \text{ between } 0 \text{ and }\beta_t^*~.
\end{align*}
This implies that $\frac{|\theta_t| }{V_t} \del{1 - |\beta^*_t| }^2 \le |\beta^*_t| \le \fr{|\theta_t| }{V_t} \del{1 + |\beta^*_t|}^2$.
Using Lemma~\ref{lemma:beta-invert} to solve the second inequality, if $V_t/|\theta_t| \ge 5$ then we obtain $|\beta_t^*| \le \frac{|\theta_t|}{V_t} + 5 \del{\frac{|\theta_t|}{V_t} }^{2}$.
Also, solving the first inequality, we get 
\begin{equation}
\label{eq:lower_bound_betastar}
|\beta^*| 
\geq \frac12\del{\frac{V_t}{|\theta_t|}+2 - \sqrt{\frac{V_t}{\theta_t}\left(\frac{V_t}{\theta_t}+4\right)}} 
=\fr{2}{\frac{V_t}{\theta_t}+2 + \sqrt{\left(\frac{V_t}{\theta_t}+2\right)^2 - 4}} \ge \fr{|\theta_t|}{2|\theta_t| + V_t}~.
\end{equation}
 
If $|\beta^*| = 1$, then there are two cases: either the absolute value of the unconstrained maximizer is 1 or it is bigger than 1. In the first case, $f'(\beta_t^*)>0$ if $\beta_t^*=1$ and $f'(\beta_t^*)<0$ if $\beta_t^*=-1$. Using the fact that $\beta_t^*$ and $\theta_t$ have the same sign, in both cases reasoning as above, we have
\begin{align*}
    \left|\sum_{i=1}^t (X_i-m)\right| -  \sum_{i=1}^t \fr{(X_i-m)^2}{(1 + \beta (X_i-m))^2} \ge 0, \text{ for some } |\beta|\leq 1 ,
\end{align*}
which implies $V_t/|\theta_t| \le 4$.

We now do a case analysis.

\textbf{Case 1. $V_t \ge 5|\theta_t|$ }

Our analysis above for the case of $|\beta_t^*| = 1$ implies that $V_t \le 4|\theta_t|$, which contradicts the condition of Case 1. Thus, $|\beta^*_t| < 1$ which from \eqref{eq:lower_bound_betastar} implies that $|\beta^*_t| \ge \frac{|\theta_t|}{2|\theta_t| + V_t}$.

Given that $\Delta_t \leq \frac{1+\min(q_t\beta^*_t,0)}{\sqrt{V_t}}$, from \eqref{eq:regret_lil_2} we obtain
\begin{align*}
  \ln\Wealth_t(\beta^*_t) - \ln\Wealth_t \le \ln\del{\frac{\sqrt{e}}{\Delta_t \cd F(\beta^*_t)} }~.
\end{align*}
So, from Ville's inequality, with probability at least $1-\delta$, we have that, $\forall t \ge 1$, 
\begin{align}
\label{eq:proof_lil_1}
  \max_{\beta\in[-1,1]} \ln\Wealth_t(\beta) \le \ln\del{\frac{\sqrt{e}}{\Delta_t \cd F(\beta^*_t)} } + \ln \frac{1}{\delta}~.
\end{align}
Now, it remains to figure out the range of $m$ that satisfies the above inequality.

Define $\hbeta^*_t:=\frac{\theta_t}{4/3 \cdot |\theta_t|+2V_t}$. Using Lemma~\ref{lemma:max_approx_wealth} and \eqref{eq:fan}, we have
\begin{align*}
  \ln \Wealth_t(\beta^*_t) 
   \ge \ln\Wealth_t(\hbeta^*_t)  
  \ge \frac{\theta^2}{4/3|\theta_t| + 2 V_t}~.
\end{align*}
We need a lower bound for $\Delta_t$. If $\Delta_t=\tilde{\Delta}_t$, then $\Delta_t\geq \frac{1-|\beta^*|}{\sqrt{V_t}}$. Instead, if $\Delta_t\neq\tilde{\Delta}_t$, then by its definition we have $\Delta_t=|\beta^*_t|$. 
\begin{align*}
  \frac{|\beta^*_t|}{1-|\beta^*_t|}  
   \sr{\eqref{eq:lower_bound_betastar}}{\ge} \frac{|\theta_t|}{|\theta_t| + V_t} 
     > \frac{\sqrt{2V_t}}{|\theta_t| + V_t}
   \ge \frac{\sqrt{2V_t}}{\frac65 V_t}
     \ge \frac{1}{\sqrt{V_t}} 
  \\ \implies \frac{1-|\beta^*_t|}{\sqrt{V_t}} < |\beta^*_t|~.
\end{align*}
Hence, in both cases we have $\Delta_t\geq \frac{1-|\beta^*|}{\sqrt{V_t}}$.
Thus, using \eqref{eq:proof_lil_1} and the lower bound of $\Delta_t$, we obtain
\begin{align*}
  \frac{\theta_t^2}{\frac{4}{3} |\theta_t| + 2 V_t}
  \le \ln\del{\frac{\sqrt{e}}\delta \cd \frac{\sqrt{V_t}}{1-|\beta^*_t|} |\beta^*_t| h(\beta^*_t) } ~.
\end{align*}

Since $|\beta^*_t| \le \frac{|\theta_t|}{V_t} + 5 \del{\frac{\theta_t}{V_t} }^2$, we have $|\beta^*_t| \le \frac{|\theta_t|}{V_t} + 5 \cd \frac{|\theta_t|}{V_t} \cd \fr15 \le \fr{2|\th_t|}{V_t} \le \fr25$.
So, using $1/(1-|\beta^*_t|) \le \fr{5}{3}$, $|\beta^*_t| \le \fr{2|\theta_t|}{V_t}$, and $h(\beta^*_t) \le h(\frac{|\theta_t|}{2|\theta_t| + V_t} )$, we have
\begin{align*}
\frac{\theta^2_t}{\frac{4}{3} |\theta_t| + 2 V_t}
&\leq \ln\del{\frac{10\sqrt{e}}{3\delta} \cdot \frac{|\theta_t|}{\sqrt{V_t}} h\left(\frac{|\theta_t|}{2|\theta_t| + V_t}\right)}
= \ln\del{\frac{10\sqrt{e}}{3\delta} \cdot \frac{|\theta_t|}{\sqrt{\frac{4}{3}|\theta|+2V_t}}\frac{\sqrt{\frac{4}{3}|\theta_t|+2V_t}}{\sqrt{V_t}} h\left(\frac{|\theta_t|}{2|\theta_t| + V_t}\right)} \\
&\leq \ln\del{\frac{20}{3\delta} \cdot \frac{|\theta_t|}{\sqrt{\frac{4}{3}|\theta_t|+2V_t}} h\left(\frac{1}{2 + \sqrt{V_t/2}}\right)}~. \tag{$|\th_t| \le \fr15 V_t$ }
\end{align*}
From the definition of $h(x)$, we have that the minimum of $h$ in $[-1,1]$ is greater than 6.
So, using Lemma~\ref{lemma:lambert_log_ineq} with $x = \frac{\theta^2_t}{\frac{4}{3} |\theta_t| + 2 V_t}$, $A = \fr12$, and $B = \ln\del{\fr{20}{3\dt} h\del{\fr{1}{2+\sqrt{V_t/2}}}} $, we have
$
\frac{\theta^2}{\frac{4}{3} |\theta_t| + 2 V_t}
\leq U_t$, 
where 
$U_t:= -\fr12 W_{-1} \left(-\frac{2}{\del{\frac{20}{3\delta} \cdot h\left(\frac{1}{2 + \sqrt{V_t/2}}\right)}^2}\right)$.
Use the fact that $V_t=V_t(\hat{\mu}_t)+\theta_t^2/t$ in $\theta_t^2 \leq (\frac{4}{3} |\theta_t|+2V_t)U_t$ to obtain $|\theta_t|\leq \frac{\frac43 U_t}{1-2U_t/t}+\sqrt{\frac{2U_t}{1-2U_t/t}V_t(\hat{\mu}_t)}$.

\textbf{Case 2: $V_t/|\theta_t| < 5 $}\\ 
In this case, either $|\beta^*_t|=1$ or $|\beta^*_t|<1$ and
from \eqref{eq:lower_bound_betastar} we have that $|\beta^*_t|\geq \frac{1}{2+\frac{V_t}{|\theta_t|}}\geq \frac{1}{7}$.
Hence, from \eqref{eq:regret_lil_1} and Ville's inequality we have
\begin{align*}
  \ln(\Wealth_t(\beta^*_t)) 
  &\le \frac12 \ln(\Wealth_t(\beta^*_t)) + \ln \fr{7}{F(1) \delta} 
\\  \implies
    \frac12 \ln(\Wealth(\beta^*_t)) 
    &\le \ln \frac{7}{F(1) \delta}
\\ \implies
    \frac{\theta_t^2}{4(|\theta_t|+V_t)}  
    &\le \ln \frac{7}{F(1) \delta} \tag{Lemma~\ref{lemma:max_approx_wealth} and \eqref{eq:fan}}
\\ \implies
    |\theta_t|  
    &< 24 \ln \frac{7}{F(1) \delta}~.  \tag{$V_t < 5|\theta_t|$ }
\end{align*}

In this case, we do not have a guarantee that the set of $m$ that satisfies the above inequalities is an interval. However, from Theorem~\ref{thm:f_weighted_intervals}, we know that if you could invert exactly the wealth inequality we would obtain intervals. Hence, given that we have derived an upper bound on the regret, any interval $[\ell_t, u_t]$ found by the algorithm is always a valid one.

Finally, let us now focus on the asymptotic behaviour of the term $\frac{U_t}{1-2U_t/t}$ that appears in \textbf{Case 1}.
Denote by $C_t=\ln\left(\frac{20}{3\delta} \cdot h\left(\frac{1}{2 + \sqrt{V_t/2}}\right)\right)$ and by $C'_t=\ln\left(\frac{20}{3\delta} \cdot h\left(\frac{1}{2 + \sqrt{t/2}}\right)\right)$.
Using the second inequality of Lemma~\ref{lemma:lambert_log_ineq} and $V_t\leq t$, we have $U_t \leq C_t + \sqrt{C_t-\frac{1}{2}}\leq 2C_t \leq 2 C'_t$.
Hence, given that and $h(x) \leq K_1 \ln^3(1/x)$ for a suitable constant $K_1$, we have $U_t=O(\ln\frac{ \ln^3 t}{\delta})$.
Now, we bound $V_t$ using the fact that $\theta_t^2
\le \frac43 |\theta_t| U_t + 2V_t U_t
\le 3V_t U_t
$, where we used the assumption $V_t\geq 5|\theta_t|$. We now obtain
$V_t = V_t(\hat{\mu}_t) + \frac1t \theta_t^2 \le V_t(\hat{\mu}_t) + \frac1t 3V_t U_t$, that implies for $t$ large enough and for a suitable universal constant $K_2$ that $V_t \leq \frac{V_t(\hat{\mu}_t)}{1-3U_t/t}\leq \frac{V_t(\hat{\mu}_t)}{1-K_2 \ln \frac{\ln^3 t}{\delta}}$.
Denoting by $C''_t=\ln\left(\frac{20}{3\delta} \cdot h\left(\frac{1}{2 + \sqrt{\frac{V_t(\hat{\mu}_t)}{2-2K_2 \ln \frac{\ln^3 t}{\delta}}}}\right)\right)$, we have $\frac{U_t}{1-2 U_t/t}\leq \frac{C''_t+\sqrt{C''_t}}{1-4C'_t/t}$ and $
\lim_{t\rightarrow\infty}  \frac{1}{\ln\left( \frac{\ln  \sqrt{V_t(\hat{\mu}_t)}}{\delta} \right)}  \frac{C''_t+\sqrt{C''_t}}{1-4C'_t/t}
= 1$.
\end{proof}

\section{Numerical Evaluation}
\label{sec:exp}

\begin{figure}[t]
    \centering
    \begin{tabular}{cc}
        \includegraphics[width=0.45\linewidth]{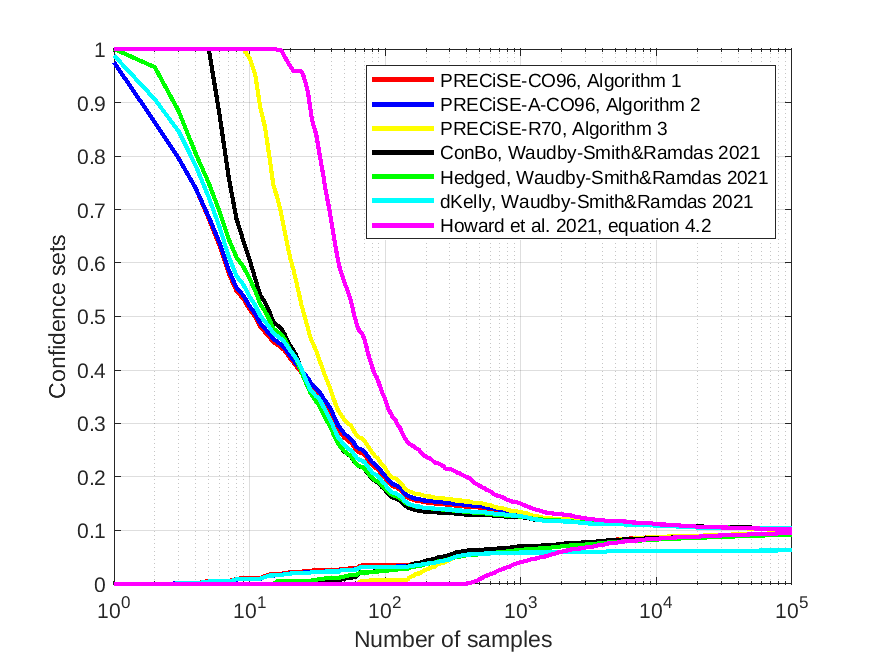} &
        \includegraphics[width=0.45\linewidth]{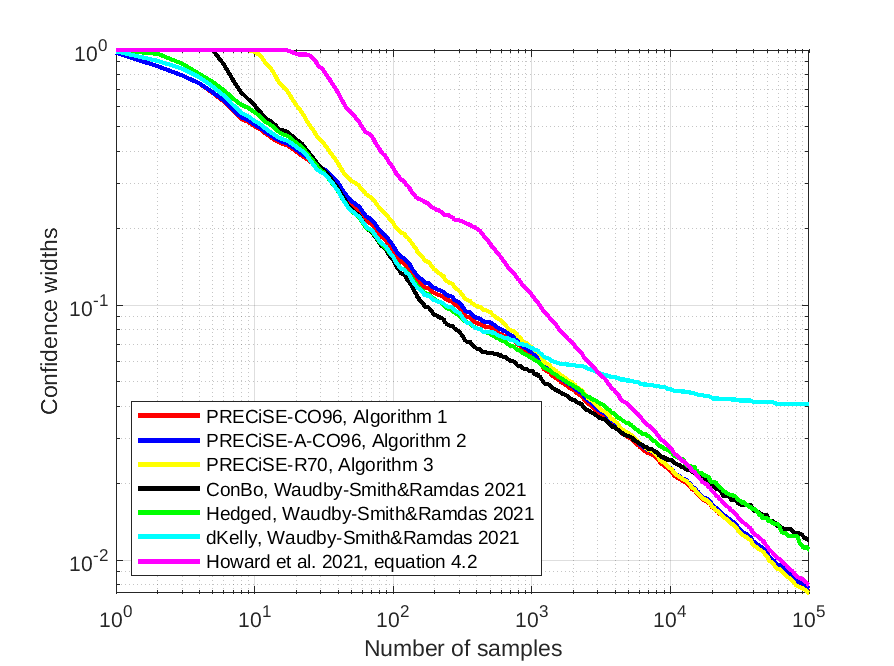}
    \end{tabular}
    \caption{Time-uniform confidence sets (left) and confidence widths (right) with $\delta=0.05$ for a sequence of i.i.d. Bernoulli(0.1).}
    \label{fig:bernoulli01}
\end{figure}

\begin{figure}[t]
    \centering
    \begin{tabular}{cc}
        \includegraphics[width=0.45\linewidth]{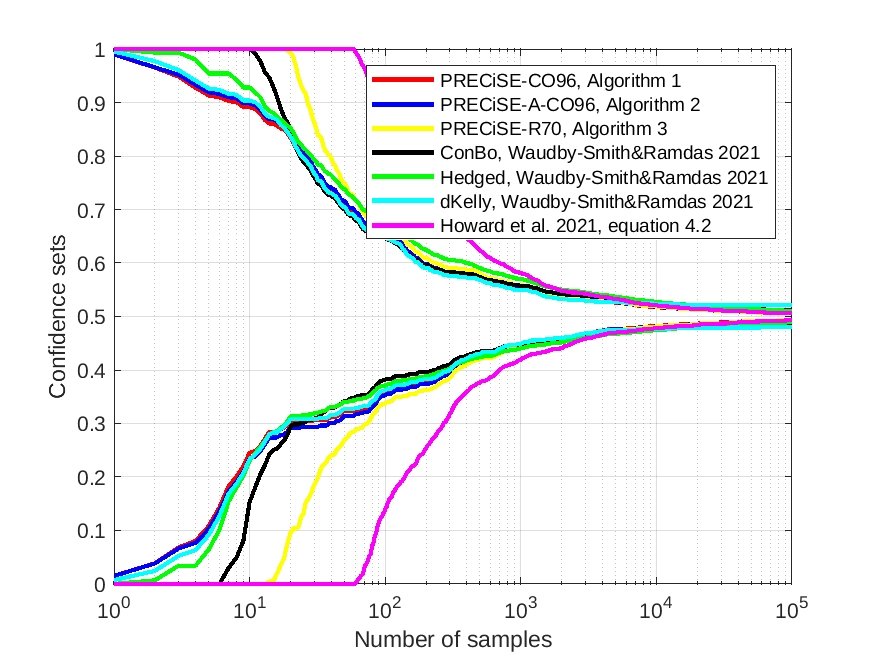} &
        \includegraphics[width=0.45\linewidth]{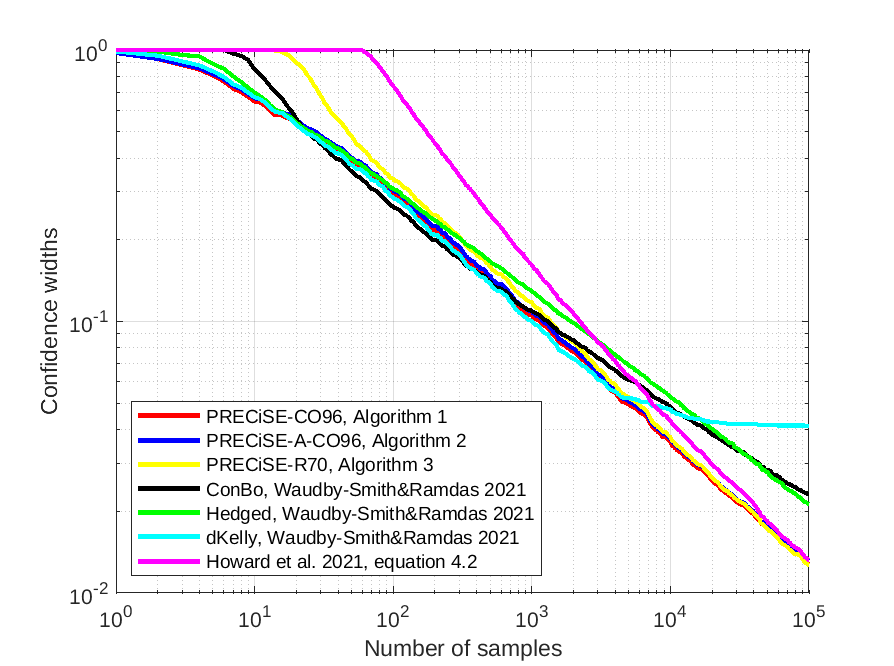}
    \end{tabular}
    \caption{Time-uniform confidence sets (left) and confidence widths (right) with $\delta=0.05$ for a sequence of i.i.d. Bernoulli(0.5).}
    \label{fig:bernoulli05}
\end{figure}

\begin{figure}[t]
    \centering
    \begin{tabular}{cc}
        \includegraphics[width=0.45\linewidth]{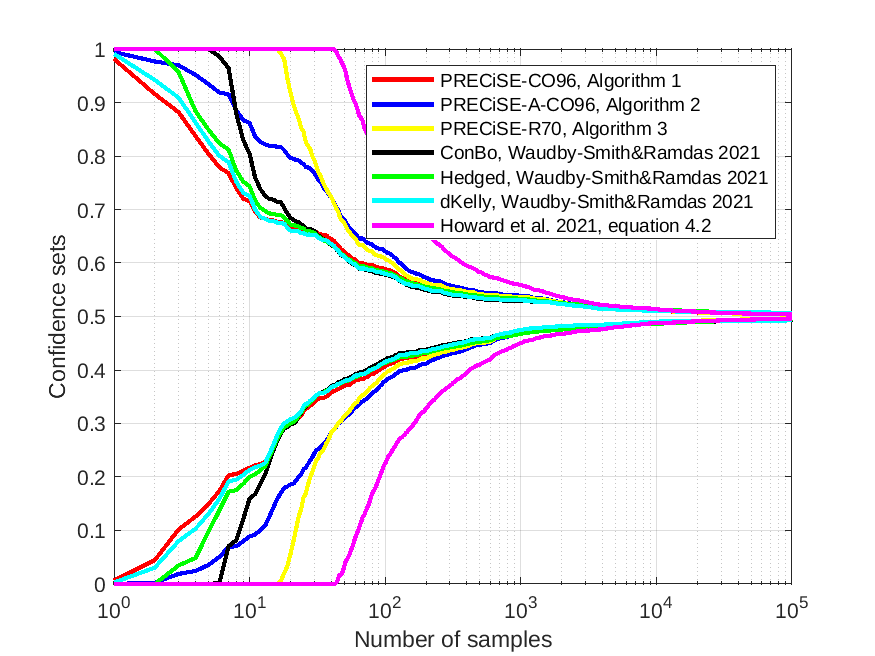} &
        \includegraphics[width=0.45\linewidth]{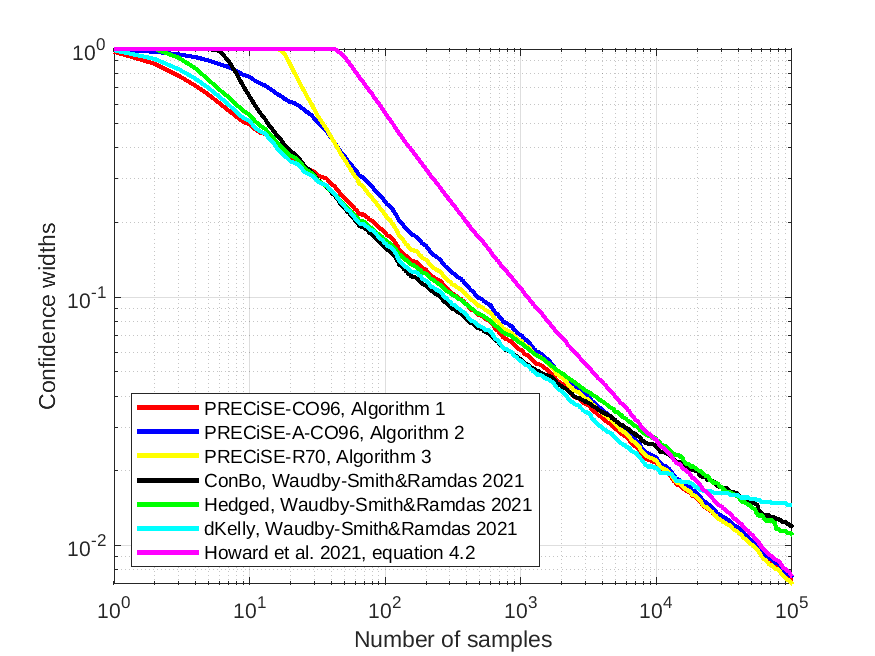}
    \end{tabular}
    \caption{Time-uniform confidence sets (left) and confidence widths (right) with $\delta=0.05$ for a sequence of i.i.d. Beta(1,1).}
    \label{fig:beta11}
\end{figure}

\begin{figure}[t]
    \centering
    \begin{tabular}{cc}
        \includegraphics[width=0.45\linewidth]{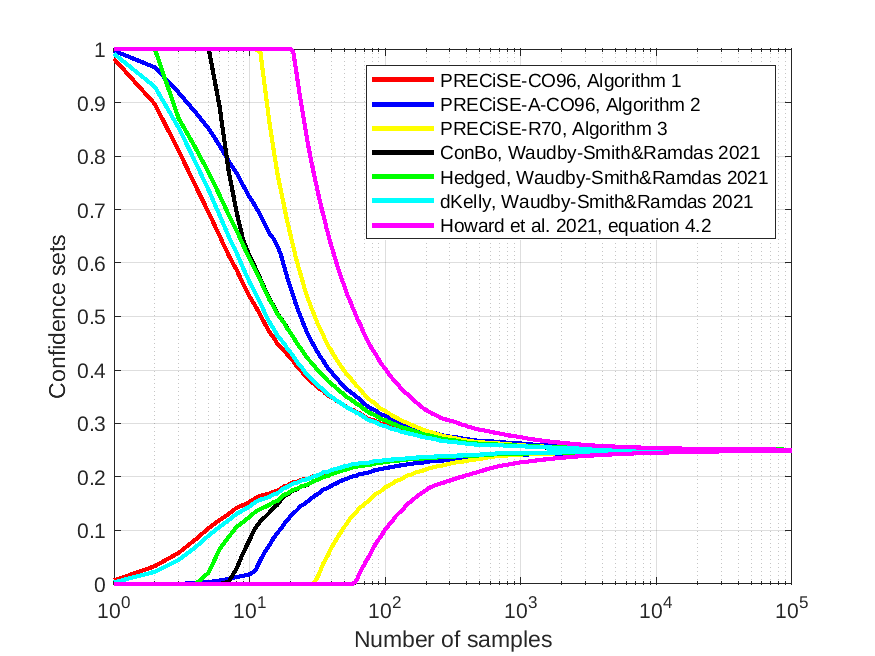} &
        \includegraphics[width=0.45\linewidth]{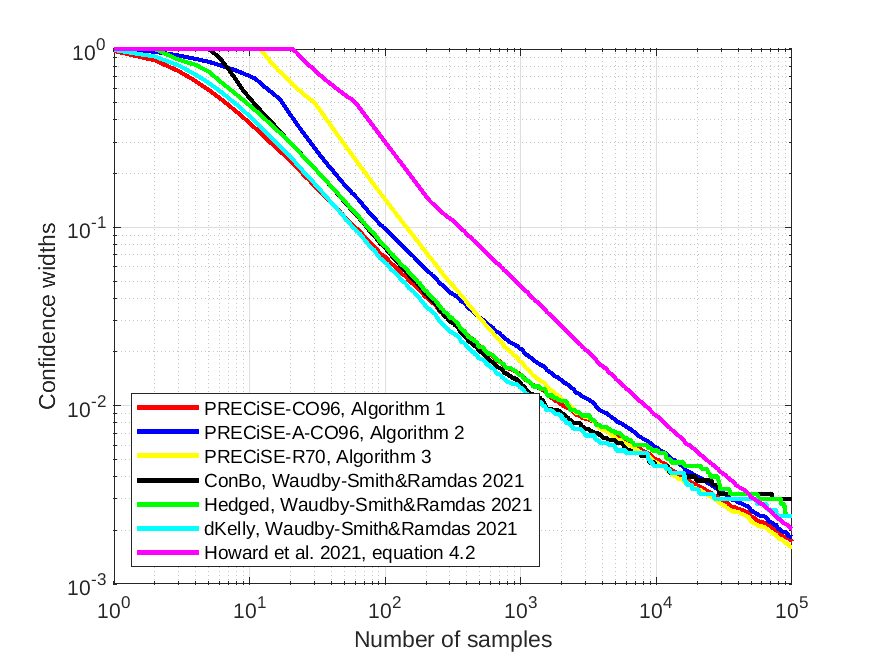}
    \end{tabular}
    \caption{Time-uniform confidence sets (left) and confidence widths (right) with $\delta=0.05$ for a sequence of i.i.d. Beta(10,30).}
    \label{fig:beta1030}
\end{figure}

To test the empirical performance of our proposed procedures, we consider the same set of experiments in \citet{Waudby-SmithR21}. They consider four different sequences of i.i.d. random variables: Bernoulli(0.1), Bernoulli(1/2), Beta(1,1), and Beta(10,30). Differently from them, we repeat each experiment 10 times and use the first 5 in which all the algorithms do not fail. In fact, all the algorithms/inequalities will fail with probability $\delta$ and averaging over their failing runs would result in confidence sequences with smaller widths. 

As baselines, we consider the numerical inversion of the betting algorithms ConBO+LBOW, Hedged, and dKelly in \citet{Waudby-SmithR21}, and the empirical Bernstein LIL concentration in \citet[Equation 4.2]{HowardRMS21} for sub-Bernoulli random variables.

We implemented all the PRECiSE algorithms in Matlab. PRECiSE-CO96 and PRECiSE-R70 use a modified Newton update with projections for lines 5 and 9, and the binary search stops when the interval is less than $10^{-4}$ Our code is available at \url{https://github.com/bremen79/precise}.
For the algorithms in \citet{Waudby-SmithR21}, we used their Python code.\footnote{https://github.com/WannabeSmith/confseq} For dKelly we used the default setting of their library of 10 different bets, while for ConBo and Hedged we used a discretization of 1000 points.

From the empirical results, we see that PRECiSE-CO96 dominates all the other methods in the small sample regime.
It is particularly striking that we can already have a confidence width smaller than 1 with only one sample, as proved in Theorem~\ref{thm:never_vacuous}. Such behavior would be impossible with an expression of the confidence interval depending only on the empirical variance.
More in details, this is possible only because we do not restrict the range of the betting fraction as in, e.g., \citet{CutkoskyO18,Waudby-SmithR21}, but instead the universal portfolio algorithm allows to consider the full interval $[-\frac{1}{1-m}, \frac{1}{m}]$. This is a due to the non-trivial feature of the universal portfolio algorithm in \citet{CoverO96} to work with unbounded loss functions. In turn, this bigger range of allowed betting fraction allows to increase as quickly as possible the wealth for values of $m$ far from the expectation with only few samples.

When the number of samples increases, the heuristic methods in \citet{Waudby-SmithR21} have a marginal gain, that seems to disappear in the large sample regime. This is probably due to the fixed discretization used in all their algorithms that does not allow to recover the correct behavior at infinity, especially in the dKelly method. Moreover, when the number of bets in dKelly goes to infinity, the asymptotic performance would be equal to the one of universal portfolio with a uniform mixture, that has a worse regret that with the Dirichelet(1/2,1/2) mixture~\citep{CoverO96}.

Finally, with a big enough number of samples the LIL confidence sequences produced by PRECiSE-R70 dominate all the other methods, even the one in \citet{HowardRMS21} at least with $10^5$ samples.

To further study the behaviour of PRECiSE-CO96 with a small number of samples, we compared it with the so-called ``exact'' confidence intervals in \citet{ClopperP34}. It is known that the confidence intervals calculated with the method in \citet{ClopperP34} are not the tightest ones \citep[see, e.g., discussion in ][]{BlythS83}, yet they are the easiest ones to compute for Bernoulli random variables.
On the other hand, the confidence intervals in \citet{ClopperP34} are not uniform over time, so they are asymptotically smaller than the tightest possible ones. Yet, PRECiSE-CO96 closely tracks their behaviour, from the very first sample and without knowledge of the distribution at hand. Hence, a practitioner might be reassured by the fact that they do not lose too much compared by the approach in \citet{ClopperP34}, yet they obtain a time-uniform guarantee that allows to decide the number of samples in a data-dependent way.

\begin{figure}[t]
    \centering
    \begin{tabular}{cc}
        \includegraphics[width=0.45\linewidth]{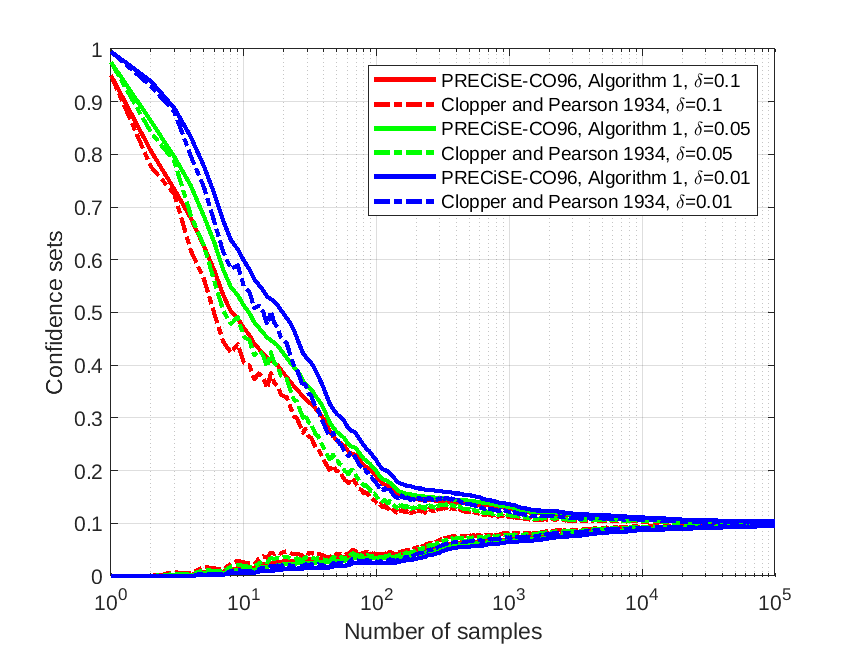} &
        \includegraphics[width=0.45\linewidth]{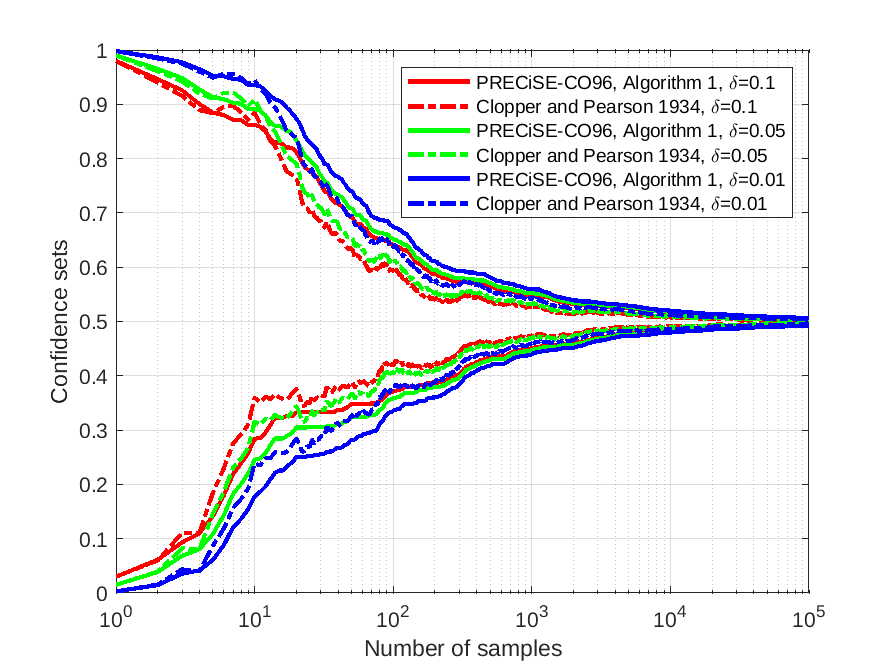}
    \end{tabular}
    \caption{Time-uniform confidence sets for PRECiSE-CO96 and confidence intervals with the ``exact'' algorithm in \citet{ClopperP34}, with $\delta=0.1, 0.05, 0.01$, for a sequence of i.i.d. Bernoulli(0.1) (left) and Bernoulli(0.5) (right).}
    \label{fig:exact}
\end{figure}

\section{Discussion}
\label{sec:disc}

We have shown that we can obtain new time-uniform concentrations and confidence sequences through a straightforward reduction from betting on a continuous coin to portfolio selection with two stocks. Numerically inverting our concentrations, gives non-vacuous confidence sequences tight with one sample and at infinity.
Here, we discuss some possible extension to this work.

\paragraph{Yet another mixture?} We have considered the Dirichlet(1/2,1/2) and the mixture in \citet{Robbins70}. However, other choices are still possible! In fact, we could choose the mixture based on $\mu$. This might seem an odd choice but it is perfectly legal. Indeed, Theorem~\ref{thm:main} uses the knowledge of $\mu$ in deciding the interval of the betting fraction.

\paragraph{Running Universal Portfolio?} One might wonder why we need to lower bound the wealth of the universal portfolio algorithm, through an upper bound to its regret, instead of just running it. Indeed, universal portfolio with two stocks with Dirichelet(1/2,1/2) and uniform mixture can be implemented with algorithm whose complexity per step is linear in the number of the received samples~\citep{CoverO96}. However, this approach would not work for the mixture in Theorem~\ref{thm:lil}.

\section*{Acknowledgments}

The authors thank Wouter Koolen for his blog post on how to prove the LIL bound that inspired this line of work in 2017, Peter Gr\"{u}nwald and Glenn Shafer for suggestions and discussion on prior work, Blair Bilodeau, Erik Learned-Miller, Erik Ordentlich, D\'{a}vid P\'{a}l, Aaditya Ramdas, Vladimir Vovk, and Ian Waudby-Smith for comments and feedback on a previous version of this manuscript.
FO is partly supported by the National Science Foundation under grant no. 2046096 ``CAREER: Parameter-free Optimization Algorithms for Machine Learning''.
KJ is partly supported by Data Science Academy and Research, Innovation \& Impact at the University of Arizona.

\bibliographystyle{plainnat}
\bibliography{learning}

\begin{thebibliography}{67}
\providecommand{\natexlab}[1]{#1}
\providecommand{\url}[1]{\texttt{#1}}
\expandafter\ifx\csname urlstyle\endcsname\relax
  \providecommand{\doi}[1]{doi: #1}\else
  \providecommand{\doi}{doi: \begingroup \urlstyle{rm}\Url}\fi

\bibitem[Abernethy et~al.(2008)Abernethy, Hazan, and Rakhlin]{AbernethyHR08}
J.~D. Abernethy, E.~Hazan, and A.~Rakhlin.
\newblock Competing in the dark: An efficient algorithm for bandit linear
  optimization.
\newblock In R.~A. Servedio and T.~Zhang, editors, \emph{Proc. of Conference on
  Learning Theory}, pages 263--274. Omnipress, 2008.

\bibitem[Adachi and Takemura(2011)]{AdachiT11}
R.~Adachi and A.~Takemura.
\newblock Sequential optimizing investing strategy with neural networks.
\newblock \emph{Expert Systems with Applications}, 38\penalty0 (10):\penalty0
  12991--12998, 2011.

\bibitem[Audibert et~al.(2009)Audibert, Munos, and
  Szepesv{\'a}ri]{AudibertMC09}
J.-Y. Audibert, R.~Munos, and C.~Szepesv{\'a}ri.
\newblock Exploration-exploitation tradeoff using variance estimates in
  multi-armed bandits.
\newblock \emph{Theoretical Computer Science}, 410\penalty0 (19):\penalty0
  1876--1902, 2009.

\bibitem[Balsubramani(2014)]{Balsubramani14}
A.~Balsubramani.
\newblock Sharp finite-time iterated-logarithm martingale concentration.
\newblock \emph{arXiv preprint arXiv:1405.2639}, 2014.

\bibitem[Barron and Xie(1996)]{BarronX96}
A.~R. Barron and Q.~Xie.
\newblock Asymptotic minimax loss for data compression, gambling, and
  prediction.
\newblock In \emph{Proceedings of the ninth annual conference on computational
  learning theory}. ACM Press, 1996.

\bibitem[Bassham et~al.(2010)Bassham, Rukhin, Soto, Nechvatal, Smid, Leigh,
  Levenson, Vangel, Heckert, and Banks]{BasshamETAL10}
L.~Bassham, A.~Rukhin, J.~Soto, J.~Nechvatal, M.~Smid, S.~Leigh, M.~Levenson,
  M.~Vangel, N.~Heckert, and D.~Banks.
\newblock A statistical test suite for random and pseudorandom number
  generators for cryptographic applications, 2010-09-16 2010.
\newblock URL
  \url{https://tsapps.nist.gov/publication/get_pdf.cfm?pub_id=906762}.

\bibitem[Blyth and Still(1983)]{BlythS83}
C.~R. Blyth and H.~A. Still.
\newblock Binomial confidence intervals.
\newblock \emph{Journal of the American Statistical Association}, 78\penalty0
  (381):\penalty0 108--116, 1983.

\bibitem[Breiman(1961)]{Breiman61}
L.~Breiman.
\newblock Optimal gambling systems for favorable games.
\newblock In \emph{Fourth Berkeley Symposium on Mathematical Statistics and
  Probability}, pages 65--78, 1961.

\bibitem[Cesa-Bianchi and Lugosi(2006)]{Cesa-BianchiL06}
N.~Cesa-Bianchi and G.~Lugosi.
\newblock \emph{Prediction, learning, and games}.
\newblock Cambridge University Press, 2006.

\bibitem[Chatzigeorgiou(2013)]{Chatzigeorgiou13}
I.~Chatzigeorgiou.
\newblock Bounds on the {Lambert} function and their application to the outage
  analysis of user cooperation.
\newblock \emph{IEEE Communications Letters}, 17\penalty0 (8):\penalty0
  1505--1508, 2013.

\bibitem[Clopper and Pearson(1934)]{ClopperP34}
C.~J. Clopper and E.~S. Pearson.
\newblock The use of confidence or fiducial limits illustrated in the case of
  the binomial.
\newblock \emph{Biometrika}, 26\penalty0 (4):\penalty0 404--413, 1934.

\bibitem[Cover(1974)]{Cover74}
T.~M Cover.
\newblock Universal gambling schemes and the complexity measures of
  {Kolmogorov} and {Chaitin}.
\newblock Technical Report~12, Department of Statistics, Stanford University,
  1974.

\bibitem[Cover(1991)]{Cover91}
T.~M. Cover.
\newblock Universal portfolios.
\newblock \emph{Mathematical Finance}, pages 1--29, 1991.

\bibitem[Cover and Ordentlich(1996)]{CoverO96}
T.~M. Cover and E.~Ordentlich.
\newblock Universal portfolios with side information.
\newblock \emph{IEEE Transactions on Information Theory}, 42\penalty0
  (2):\penalty0 348--363, 1996.

\bibitem[Cover and Thomas(2006)]{CoverT06}
T.~M. Cover and Joy~A. Thomas.
\newblock \emph{Elements of information theory}.
\newblock John Wiley \& Sons, 2006.

\bibitem[Cutkosky and Orabona(2018)]{CutkoskyO18}
A.~Cutkosky and F.~Orabona.
\newblock Black-box reductions for parameter-free online learning in {Banach}
  spaces.
\newblock In \emph{Proc. of the Conference on Learning Theory}, 2018.

\bibitem[Darling and Robbins(1967)]{DarlingR67}
D.~A. Darling and H.~Robbins.
\newblock Iterated logarithm inequalities.
\newblock \emph{Proceedings of the National Academy of Sciences of the United
  States of America}, 57\penalty0 (5):\penalty0 1188, 1967.

\bibitem[Davisson(1973)]{Davisson73}
L.~Davisson.
\newblock Universal noiseless coding.
\newblock \emph{IEEE Transactions on Information Theory}, 19\penalty0
  (6):\penalty0 783--795, 1973.

\bibitem[Fan et~al.(2015)Fan, Grama, and Liu]{FanGL15}
X.~Fan, I.~Grama, and Q.~Liu.
\newblock Exponential inequalities for martingales with applications.
\newblock \emph{Electronic Journal of Probability}, 20:\penalty0 1--22, 2015.

\bibitem[Feder et~al.(1992)Feder, Merhav, and Gutman]{FederMG92}
M.~Feder, N.~Merhav, and M.~Gutman.
\newblock Universal prediction of individual sequences.
\newblock \emph{IEEE transactions on Information Theory}, 38\penalty0
  (4):\penalty0 1258--1270, 1992.

\bibitem[Freund(1996)]{Freund96}
Y.~Freund.
\newblock Predicting a binary sequence almost as well as the optimal biased
  coin.
\newblock In \emph{Proc. of the Conference on Computational learning theory},
  pages 89--98, 1996.

\bibitem[G{\'a}cs(2005)]{Gacs05}
P.~G{\'a}cs.
\newblock Uniform test of algorithmic randomness over a general space.
\newblock \emph{Theoretical Computer Science}, 341\penalty0 (1-3):\penalty0
  91--137, 2005.

\bibitem[Garivier and Capp{\'{e}}(2011)]{garivier11kl}
A.~Garivier and O.~Capp{\'{e}}.
\newblock The {KL-UCB} algorithm for bounded stochastic bandits and beyond.
\newblock In \emph{Proceedings of the Conference on Learning Theory (COLT)},
  pages 359--376, 2011.

\bibitem[Gr{\"u}nwald et~al.(2019)Gr{\"u}nwald, de~Heide, and
  Koolen]{GrunwaldHK19}
P.~Gr{\"u}nwald, R.~de~Heide, and W.~M. Koolen.
\newblock Safe testing.
\newblock \emph{arXiv preprint arXiv:1906.07801}, 2019.

\bibitem[Haussler and Barron(1992)]{HausslerB92}
D.~Haussler and A.~Barron.
\newblock How well do {Bayes} methods work for on-line prediction of $\{\pm
  1\}$ values.
\newblock In \emph{Proceedings of the Third NEC Symposium on Computation and
  Cognition. SIAM}, 1992.

\bibitem[Haussler et~al.(1998)Haussler, Kivinen, and Warmuth]{HausslerKW98}
D.~Haussler, J.~Kivinen, and M.~K. Warmuth.
\newblock Sequential prediction of individual sequences under general loss
  functions.
\newblock \emph{IEEE Transactions on Information Theory}, 44\penalty0
  (5):\penalty0 1906--1925, 1998.

\bibitem[Hazan and Kale(2008)]{HazanK08}
E.~Hazan and S.~Kale.
\newblock Extracting certainty from uncertainty: Regret bounded by variation in
  costs.
\newblock In \emph{Proc. of the Conference on Learning Theory}, 2008.

\bibitem[Hendriks(2018)]{Hendriks18}
H.~Hendriks.
\newblock Test martingales for bounded random variables.
\newblock \emph{arXiv preprint arXiv:2109.08923}, 2018.

\bibitem[Howard et~al.(2021)Howard, Ramdas, McAuliffe, and Sekhon]{HowardRMS21}
S.~R. Howard, A.~Ramdas, J.~McAuliffe, and J.~Sekhon.
\newblock Time-uniform, nonparametric, nonasymptotic confidence sequences.
\newblock \emph{The Annals of Statistics}, 49\penalty0 (2):\penalty0
  1055--1080, 2021.

\bibitem[Jun and Orabona(2019)]{JunO19}
K.-S. Jun and F.~Orabona.
\newblock Parameter-free online convex optimization with sub-exponential noise.
\newblock In \emph{Proc. of the Conference on Learning Theory}, 2019.

\bibitem[Kalai and Vempala(2002)]{KalaiV02}
A.~T. Kalai and S.~Vempala.
\newblock Efficient algorithms for universal portfolios.
\newblock \emph{Journal of Machine Learning Research}, pages 423--440, 2002.

\bibitem[Kelly(1956)]{Kelly56}
J.~L. Kelly, jr.
\newblock A new interpretation of information rate.
\newblock \emph{IRE Transactions on Information Theory}, 2\penalty0
  (3):\penalty0 185--189, 1956.

\bibitem[Krichevsky and Trofimov(1981)]{KrichevskyT81}
R.~Krichevsky and V.~Trofimov.
\newblock The performance of universal encoding.
\newblock \emph{IEEE Trans. on Information Theory}, 27\penalty0 (2):\penalty0
  199--207, 1981.

\bibitem[Kumon et~al.(2008)Kumon, Takemura, and Takeuchi]{KumonTT08}
M.~Kumon, A.~Takemura, and K.~Takeuchi.
\newblock Capital process and optimality properties of a {Bayesian Skeptic} in
  coin-tossing games.
\newblock \emph{Stochastic Analysis and Applications}, 26\penalty0
  (6):\penalty0 1161--1180, 2008.

\bibitem[Kumon et~al.(2011)Kumon, Takemura, and Takeuchi]{KumonTT11}
M.~Kumon, A.~Takemura, and K.~Takeuchi.
\newblock Sequential optimizing strategy in multi-dimensional bounded
  forecasting games.
\newblock \emph{Stochastic Processes and their Applications}, 121\penalty0
  (1):\penalty0 155--183, 2011.

\bibitem[Levin(1976)]{Levin76}
L.~A. Levin.
\newblock Uniform tests of randomness.
\newblock \emph{Doklady Akademii Nauk}, 227\penalty0 (1):\penalty0 33--35,
  1976.
\newblock English version available at
  \url{https://www.cs.bu.edu/fac/lnd/dvi/rnd76.pdf}.

\bibitem[Li and Hoi(2018)]{LiH18}
B.~Li and S.~C.~H. Hoi.
\newblock \emph{Online portfolio selection: principles and algorithms}.
\newblock Crc Press, 2018.

\bibitem[Maurer and Pontil(2009)]{MaurerP09}
A.~Maurer and M.~Pontil.
\newblock Empirical {Bernstein} bounds and sample variance penalization.
\newblock In \emph{Proc. of the Conference on Learning Theory}, 2009.

\bibitem[Maurer(1992)]{Maurer92}
U.~M. Maurer.
\newblock A universal statistical test for random bit generators.
\newblock \emph{Journal of cryptology}, 5\penalty0 (2):\penalty0 89--105, 1992.

\bibitem[McMahan and Orabona(2014)]{McMahanO14}
H.~B. McMahan and F.~Orabona.
\newblock Unconstrained online linear learning in {Hilbert} spaces: Minimax
  algorithms and normal approximations.
\newblock In \emph{Proc of the Annual Conference on Learning Theory, COLT},
  2014.

\bibitem[Meinshausen et~al.(2009)Meinshausen, Meier, and
  B{\"u}hlmann]{MeinshausenMB09}
N.~Meinshausen, L.~Meier, and P.~B{\"u}hlmann.
\newblock P-values for high-dimensional regression.
\newblock \emph{Journal of the American Statistical Association}, 104\penalty0
  (488):\penalty0 1671--1681, 2009.

\bibitem[Mhammedi and Koolen(2020)]{MhammediK20}
Z.~Mhammedi and W.~M Koolen.
\newblock Lipschitz and comparator-norm adaptivity in online learning.
\newblock In \emph{Conference on Learning Theory}, pages 2858--2887. PMLR,
  2020.

\bibitem[Orabona(2019)]{Orabona19}
F.~Orabona.
\newblock A modern introduction to online learning.
\newblock \emph{arXiv preprint arXiv:1912.13213}, 2019.

\bibitem[Orabona and P\'al(2016)]{OrabonaP16}
F.~Orabona and D.~P\'al.
\newblock Coin betting and parameter-free online learning.
\newblock In D.~D. Lee, M.~Sugiyama, U.~V. Luxburg, I.~Guyon, and R.~Garnett,
  editors, \emph{Advances in Neural Information Processing Systems 29}, pages
  577--585. Curran Associates, Inc., 2016.

\bibitem[Phan et~al.(2021)Phan, Thomas, and Learned-Miller]{PhanTLM21}
M.~Phan, P.~Thomas, and E.~Learned-Miller.
\newblock Towards practical mean bounds for small samples.
\newblock In \emph{International Conference on Machine Learning}, pages
  8567--8576. PMLR, 2021.

\bibitem[Rakhlin and Sridharan(2017)]{RakhlinS17}
A.~Rakhlin and K.~Sridharan.
\newblock On equivalence of martingale tail bounds and deterministic regret
  inequalities.
\newblock In \emph{Proc. of the Conference On Learning Theory}, pages
  1704--1722, 2017.

\bibitem[Robbins(1970)]{Robbins70}
H.~Robbins.
\newblock Statistical methods related to the law of the iterated logarithm.
\newblock \emph{The Annals of Mathematical Statistics}, 41\penalty0
  (5):\penalty0 1397--1409, 1970.

\bibitem[Robbins and Siegmund(1974)]{RobbinsS74}
H.~Robbins and D.~Siegmund.
\newblock The expected sample size of some tests of power one.
\newblock \emph{The Annals of Statistics}, 2\penalty0 (3):\penalty0 415--436,
  1974.

\bibitem[Rukhin(2000)]{Rukhin00}
A.~L. Rukhin.
\newblock Approximate entropy for testing randomness.
\newblock \emph{Journal of Applied Probability}, 37\penalty0 (1):\penalty0
  88--100, 2000.

\bibitem[Ryabko and Monarev(2005)]{RyabkoM05}
B.~Ya. Ryabko and V.~A. Monarev.
\newblock Using information theory approach to randomness testing.
\newblock \emph{Journal of Statistical Planning and Inference}, 133\penalty0
  (1):\penalty0 95--110, 2005.

\bibitem[Schnorr(1971)]{Schnorr71}
C.-P. Schnorr.
\newblock A unified approach to the definition of random sequences.
\newblock \emph{Mathematical systems theory}, 5\penalty0 (3):\penalty0
  246--258, 1971.

\bibitem[Shafer(2021)]{Shafer21}
G.~Shafer.
\newblock Testing by betting: A strategy for statistical and scientific
  communication.
\newblock \emph{Journal of the Royal Statistical Society: Series A (Statistics
  in Society)}, 184\penalty0 (2):\penalty0 407--431, 2021.

\bibitem[Shafer and Vovk(2001)]{ShaferV05}
G.~Shafer and V.~Vovk.
\newblock \emph{Probability and finance: it's only a game!}
\newblock John Wiley \& Sons, 2001.

\bibitem[Shafer and Vovk(2019)]{ShaferV19}
G.~Shafer and V.~Vovk.
\newblock \emph{Game-Theoretic Foundations for Probability and Finance}, volume
  455.
\newblock John Wiley \& Sons, 2019.

\bibitem[Shafer et~al.(2011)Shafer, Shen, Vereshchagin, and Vovk]{ShaferSVS11}
G.~Shafer, A.~Shen, N.~Vereshchagin, and V.~Vovk.
\newblock Test martingales, {Bayes} factors and $p$-values.
\newblock \emph{Statistical Science}, 26\penalty0 (1):\penalty0 84--101, 2011.

\bibitem[Shalev-Shwartz(2007)]{Shalev-Shwartz07}
S.~Shalev-Shwartz.
\newblock \emph{Online Learning: Theory, Algorithms, and Applications}.
\newblock PhD thesis, The Hebrew University, 2007.

\bibitem[Shtar'kov(1987)]{Shtarkov87}
Y.~M. Shtar'kov.
\newblock Universal sequential coding of single messages.
\newblock \emph{Problemy Peredachi Informatsii}, 23\penalty0 (3):\penalty0
  3--17, 1987.

\bibitem[van~der Hoeven(2019)]{vanderHoeven19}
D.~van~der Hoeven.
\newblock User-specified local differential privacy in unconstrained adaptive
  online learning.
\newblock In H.~Wallach, H.~Larochelle, A.~Beygelzimer, F.~d'Alch\'{e} Buc,
  E.~Fox, and R.~Garnett, editors, \emph{Advances in Neural Information
  Processing Systems}, volume~32. Curran Associates, Inc., 2019.

\bibitem[Ville(1939)]{Ville39}
J.~Ville.
\newblock \emph{Étude critique de la notion de collectif}.
\newblock Gauthier-Villars, Paris, 1939.

\bibitem[Vovk(2007)]{Vovk07}
V.~Vovk.
\newblock Hoeffding's inequality in game-theoretic probability.
\newblock \emph{arXiv preprint arXiv:0708.2502}, 2007.

\bibitem[Vovk and Watkins(1998)]{VovkW98}
V.~Vovk and C.~Watkins.
\newblock Universal portfolio selection.
\newblock In \emph{Proc. of the Conference on Computational Learning Theory},
  pages 12--23, 1998.

\bibitem[Vovk(1990)]{Vovk90}
V.~G. Vovk.
\newblock Aggregating strategies.
\newblock \emph{Proc. of the Conference on Computational Learning Theory},
  pages 371--386, 1990.

\bibitem[Vovk(1993)]{Vovk93}
V.~G. Vovk.
\newblock A logic of probability, with application to the foundations of
  statistics.
\newblock \emph{Journal of the Royal Statistical Society. Series B
  (Methodological)}, 55\penalty0 (2):\penalty0 317--351, 1993.

\bibitem[Wald(1945)]{Wald45}
A.~Wald.
\newblock Sequential tests of statistical hypotheses.
\newblock \emph{The annals of mathematical statistics}, 16\penalty0
  (2):\penalty0 117--186, 1945.

\bibitem[Waudby-Smith and Ramdas(2021)]{Waudby-SmithR21}
I.~Waudby-Smith and A.~Ramdas.
\newblock Estimating means of bounded random variables by betting.
\newblock \emph{arXiv preprint arXiv:2010.09686}, 2021.

\bibitem[Xie and Barron(1997)]{XieB97}
Q.~Xie and A.~R. Barron.
\newblock Minimax redundancy for the class of memoryless sources.
\newblock \emph{IEEE Transactions on Information Theory}, 43\penalty0
  (2):\penalty0 646--657, 1997.

\bibitem[Ziv(1990)]{Ziv90}
J~Ziv.
\newblock Compression, tests for randomness and estimating the statistical
  model of an individual sequence.
\newblock In \emph{Sequences}, pages 366--373. Springer, 1990.

\end{thebibliography}

\end{document}